\documentclass{article}

\usepackage[accepted]{icml2024}

\usepackage{hyperref}       
\usepackage{url}            
\usepackage{booktabs}       
\usepackage{amsfonts}       
\usepackage{nicefrac}       
\usepackage{microtype}      
\usepackage{lipsum}		
\usepackage{graphicx}
\usepackage{natbib}
\usepackage{doi}
\usepackage{multirow} 
\usepackage{makecell}

\usepackage{amsmath,amsthm,amsfonts,amssymb, mathtools,bm}
\usepackage{algorithm}
\usepackage{algorithmic}
\usepackage[capitalize,noabbrev]{cleveref}
\usepackage{tikz}
\usepackage{pgfplots}
\usepackage{quiver} 
\usepackage{tikz}
\usepackage{graphicx}
\usepackage{xspace}
\usetikzlibrary{shadings,decorations.pathreplacing, calc,shapes,decorations}
\usepackage{caption}
\usepackage{subcaption}
\usepackage[shortlabels]{enumitem}
\usepackage{xfrac}
\usetikzlibrary{arrows.meta}

\usepackage{subcaption} 

\usepackage{xcolor}         
\usepackage{wrapfig}        
\usepackage{makecell}       

\usetikzlibrary{decorations.pathreplacing, calligraphy}

\theoremstyle{plain}
\newtheorem{theorem}{Theorem}[section]
\newtheorem{proposition}[theorem]{Proposition}
\newtheorem{lemma}[theorem]{Lemma}

\theoremstyle{definition}

\theoremstyle{remark}

\theoremstyle{plain}

\newtheorem{innercustomprop}{Proposition}
\newenvironment{propositioncopy}[1]{\renewcommand\theinnercustomprop{#1}\innercustomprop}{\endinnercustomprop}

\newcommand{\cognn}{\textsc{Co-GNN}\xspace}
\newcommand{\cognns}{\textsc{Co-GNN}\text{s}\xspace}
\newcommand{\sumgnn}{\textsc{SumGNN}\xspace}
\newcommand{\meangnn}{\textsc{MeanGNN}\xspace}
\newcommand{\sumgnns}{\textsc{SumGNN}\text{s}\xspace}
\newcommand{\meangnns}{\textsc{MeanGNN}\text{s}\xspace}

\newcommand{\cosum}{\textsc{Co-GNN}(\Sigma,\Sigma)\xspace}
\newcommand{\comean}{\textsc{Co-GNN}(\mu,\mu)\xspace}
\newcommand{\comeansum}{\textsc{Co-GNN}(\Sigma,\mu)\xspace}
\newcommand{\cogatsum}{\textsc{Co-GNN}(\Sigma,\alpha)\xspace}
\newcommand{\cogin}{\textsc{Co-GNN}(\epsilon,\epsilon)\xspace}
\newcommand{\cogcn}{\textsc{Co-GNN}(*,*)\xspace}

\newcommand{\rootn}{\textsc{RootNeighbors}\xspace}
\newcommand{\cycles}{\textsc{Cycles}\xspace}

\newcommand{\stdfont}{\fontsize{8pt}{9}\selectfont}

\newcommand\stepl{{\left(\ell\right)}}
\newcommand\steplplus{{\left(\ell + 1\right)}}

\newcommand{\isolatelong}{\textsc{Isolate}\xspace}
\newcommand{\listenlong}{\textsc{Listen}\xspace}
\newcommand{\broadcastlong}{\textsc{Broadcast}\xspace}
\newcommand{\standardlong}{\textsc{Standard}\xspace}

\newcommand{\isolate}{\textsc{I}\xspace}
\newcommand{\listen}{\textsc{L}\xspace}
\newcommand{\broadcast}{\textsc{B}\xspace}
\newcommand{\standard}{\textsc{S}\xspace}

\newcommand{\env}{{\eta}}
\newcommand{\action}{{\pi}}

\newcommand\red[1]{\textcolor{red}{#1}}  
\newcommand\blue[1]{\textcolor{blue}{#1}}  
\newcommand\gray[1]{\textcolor{gray}{#1}}  
\newcommand\draft[1]{\textcolor{blue}{#1}}  

\usepackage{amsmath,amsfonts,bm}

\def\eqref#1{equation~\ref{#1}}

\def\1{\bm{1}}

\def\vg{{\bm{g}}}
\def\vh{{\bm{h}}}

\def\vp{{\bm{p}}}
\def\vq{{\bm{q}}}

\def\vx{{\bm{x}}}

\def\vz{{\bm{z}}}

\def\mW{{\bm{W}}}
\def\mX{{\bm{X}}}

\DeclareMathAlphabet{\mathsfit}{\encodingdefault}{\sfdefault}{m}{sl}
\SetMathAlphabet{\mathsfit}{bold}{\encodingdefault}{\sfdefault}{bx}{n}

\def\sR{{\mathbb{R}}}

\icmltitlerunning{Cooperative Graph Neural Networks}

\begin{document}

\twocolumn[
\icmltitle{Cooperative Graph Neural Networks}



\icmlsetsymbol{equal}{*}

\begin{icmlauthorlist}
\icmlauthor{Ben Finkelshtein}{oxford}
\icmlauthor{Xingyue Huang}{oxford}
\icmlauthor{Michael Bronstein}{oxford}
\icmlauthor{{\.I}smail {\.I}lkan Ceylan}{oxford}
\end{icmlauthorlist}

\icmlaffiliation{oxford}{Department of Computer Science, University of Oxford}

\icmlcorrespondingauthor{name surname}{\{name.surname\}@cs.ox.ac.uk}

\icmlkeywords{graph neural networks, dynamic message passing, information flow}

\vskip 0.3in
]



\printAffiliationsAndNotice{}  

\begin{abstract}
Graph neural networks are popular architectures for graph machine learning, based on iterative computation of node representations of an input graph through a series of invariant transformations. A large class of graph neural networks follow a standard message-passing paradigm: at every layer, each node state is updated based on an aggregate of messages from its neighborhood. In this work, we propose a novel framework for training graph neural networks, where every node is viewed as a \emph{player} that can choose to either `listen’, `broadcast’, `listen and broadcast', or to `isolate’. The standard message propagation scheme can then be viewed as a special case of this framework where every node `listens and broadcasts' to all neighbors. Our approach offers a more flexible and dynamic message-passing paradigm, where each node can determine its own strategy based on their state,  effectively exploring the graph topology while learning. We provide a theoretical analysis of the new message-passing scheme which is further supported by an extensive empirical analysis on synthetic and real-world data.
\end{abstract}

\section{Introduction}

Graph neural networks (GNNs) \citep{Scarselli09,Gori2005} are a class of architectures for learning on graph-structured data. Their success in various graph machine learning (ML) tasks~\citep{Shlomi21,DuvenaudMABHAA15,ZitnikAL18} has led to a surge of different architectures~\citep{Kipf16,xu18,velic2018graph,hamilton2017inductive}. 
The vast majority of GNNs can be implemented through \emph{message-passing}, where the fundamental idea is to update each node's representation based on an aggregate of messages flowing from the node's neighbors \cite{GilmerSRVD17}. 

The message-passing paradigm has been very influential in graph ML, but it also comes with well-known limitations related to the information flow on a graph, pertaining to \emph{long-range} dependencies \cite{dwivedi2023long}. 
In order to receive information from $k$-hop neighbors, a network needs at least $k$ layers, which typically implies an exponential growth of a node's receptive field. The growing amount of information needs then to be compressed into fixed-sized node embeddings, possibly leading to information loss, referred to as {\em over-squashing}~\citep{Alon-ICLR21}.  Another well-known limitation related the information flow is \emph{over-smoothing} \cite{LiHW18}: the node features can become increasingly similar as the number of layers increases.

\textbf{Motivation.} Our goal is to generalize the message-passing scheme by allowing each node to decide how to propagate information {\em from} or {\em to} its neighbors, thus enabling a more flexible flow of information. 
Consider the example depicted in \Cref{fig:layerwise-flow}, where the top row shows the information flow relative to the node $u$ across three layers, and the bottom row shows the information flow relative to the node $v$ across three layers. Node $u$ listens to every neighbor in the first layer, only to $v$ in the second layer, and to nodes $s$ and $r$ in the last layer. On the other hand, node $v$ listens to node $w$ for the first two layers, and to node $u$ in the last layer. To realize this scenario, each node should be able to decide whether or not to listen to a particular node at each layer: a {\em dynamic} and {\em asynchronous} message-passing scheme, which clearly falls outside of standard message-passing. 

\textbf{Approach.}
To achieve this goal, we regard each node as a {\em player} that can take the following actions in each layer: 
\begin{enumerate}[noitemsep,topsep=0pt,parsep=1pt,partopsep=0pt,leftmargin=*, label={\small \textbullet}]
    \item \standardlong(\standard): Broadcast to neighbors that listen \emph{and} listen to neighbors that broadcast.
    \item \listenlong(\listen): Listen to neighbors that broadcast.
    \item \broadcastlong(\broadcast): Broadcast to neighbors that listen.
    \item  \isolatelong(\isolate): Neither listen nor broadcast.
\end{enumerate}

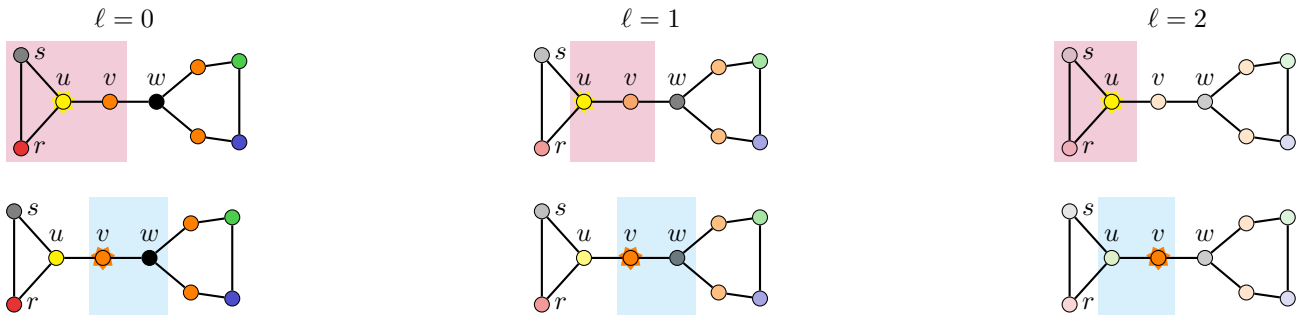
\begin{figure*}[t!]
    \centering
    \begin{tikzpicture} 	 	
        \definecolor{blue}{rgb}{0.3,0.3, 0.8}
        \definecolor{green}{rgb}{0.3,0.8, 0.3}
        \definecolor{red}{rgb}{0.9,0.2, 0.2}
        \definecolor{gray3}{rgb}{0.9,0.9,.9}
        \definecolor{gray2}{rgb}{0.3,0.3,0.3}    
    
        \node[] (l1) at (-7, 0) {$\ell=0$};  
        \node[] (l2) at (0, 0) {$\ell=1$};  
        \node[] (l3) at (7, 0) {$\ell=2$};  
    \end{tikzpicture}
    \hspace{7em}
    \vspace{.3cm}
    \begin{tikzpicture}[
        vertex/.style = {draw,circle,inner sep=2pt, minimum height= 5pt},
        directed/.style = {thick, black!100, >=stealth, ->},
        bi/.style = {thick,black!100, stealth-stealth},
        undirected/.style = {thick,black!100, >=stealth,-}
        ]
        
        \definecolor{blue}{rgb}{0.3,0.3, 0.8}
        \definecolor{green}{rgb}{0.3,0.8, 0.3}
        \definecolor{red}{rgb}{0.9,0.2, 0.2}
        \definecolor{gray3}{rgb}{0.9,0.9,.9}
        \definecolor{gray2}{rgb}{0.3,0.3,0.3}

        \draw[draw=none, fill=purple!50,fill opacity=0.4] (-0.4,0.8) rectangle (-2.0,-0.8); 
        
        \node[] (s) at (-1.55, 0.65) {$s$};
        \node[] (r) at (-1.55, -0.6) {$r$};
        \node[] (u) at (-1.24, 0.3) {$u$};  
        \node[] (v) at (-0.62, 0.3) {$v$};
        \node[] (w) at (0, 0.3) {$w$};
        
        \node[vertex, fill=red] (r) at (-1.8, -0.62) {};
        \node[vertex, fill=gray] (s) at (-1.8, 0.62) {};
        \node[star,star points=7,star point ratio=0.7, fill=yellow] (u) at (-1.24, 0) {};
        \node[vertex] (u) at (-1.24, 0) {};
        \node[vertex, fill=orange] (v) at (-0.62, 0) {};
        \node[vertex, fill=black] (w) at (0, 0) {};
        \node[vertex, fill=orange] (w_top_right) at (0.55, 0.46) {};
        \node[vertex, fill=orange] (w_bot_right) at (0.55, -0.46) {};
            \node[vertex, fill=green] (rightmost_top) at (1.1, 0.54) {};
        \node[vertex, fill=blue] (rightmost_bot) at (1.1, -0.54) {};  
        
        \draw[undirected] (s) edge (r); 
        \draw[undirected] (s) edge (u); 
        \draw[undirected] (r) edge (u); 
        \draw[undirected] (u) edge (v);
        \draw[undirected] (v) edge (w); 
        \draw[undirected] (w) edge (w_top_right);
        \draw[undirected] (w) edge (w_bot_right); 
        \draw[undirected] (w_top_right) edge (rightmost_top); 
        \draw[undirected] (w_bot_right) edge (rightmost_bot); 
        \draw[undirected] (rightmost_top) edge (rightmost_bot);
    \end{tikzpicture}
    \hfill
    \begin{tikzpicture}[
        vertex/.style = {draw,circle,inner sep=2pt, minimum height= 5pt, fill opacity=0.5},
        directed/.style = {thick, black!100, >=stealth, ->},
        bi/.style = {thick,black!100, stealth-stealth},
        undirected/.style = {thick,black!100, >=stealth,-}
        ]
        
        \definecolor{blue}{rgb}{0.3,0.3, 0.8}
        \definecolor{green}{rgb}{0.3,0.8, 0.3}
        \definecolor{red}{rgb}{0.9,0.2, 0.2}
        \definecolor{gray3}{rgb}{0.9,0.9,.9}
        \definecolor{gray2}{rgb}{0.3,0.3,0.3}

        \draw[draw=none, fill=purple!50,fill opacity=0.4] (-0.3,0.8) rectangle (-1.42,-0.8); 
        
        \node[] (s) at (-1.55, 0.65) {$s$};
        \node[] (r) at (-1.55, -0.6) {$r$};
        \node[] (u) at (-1.24, 0.3) {$u$};  
        \node[] (v) at (-0.62, 0.3) {$v$};
        \node[] (w) at (0, 0.3) {$w$};
        
        \node[vertex, fill=red] (r) at (-1.8, -0.62) {};
        \node[vertex, fill=gray] (s) at (-1.8, 0.62) {};
        \node[star,star points=7,star point ratio=0.6, fill=yellow] (u) at (-1.24, 0) {};
        \node[vertex, fill=yellow] (u) at (-1.24, 0) {};
        \node[vertex, fill=orange] (v) at (-0.62, 0) {};
        \node[vertex, fill=black] (w) at (0, 0) {};
        \node[vertex, fill=orange] (w_top_right) at (0.55, 0.46) {};
        \node[vertex, fill=orange] (w_bot_right) at (0.55, -0.46) {};
            \node[vertex, fill=green] (rightmost_top) at (1.1, 0.54) {};
        \node[vertex, fill=blue] (rightmost_bot) at (1.1, -0.54) {};  
        
        \draw[undirected] (s) edge (r); 
        \draw[undirected] (s) edge (u); 
        \draw[undirected] (r) edge (u); 
        \draw[undirected] (u) edge (v);
        \draw[undirected] (v) edge (w); 
        \draw[undirected] (w) edge (w_top_right);
        \draw[undirected] (w) edge (w_bot_right); 
        \draw[undirected] (w_top_right) edge (rightmost_top); 
        \draw[undirected] (w_bot_right) edge (rightmost_bot); 
        \draw[undirected] (rightmost_top) edge (rightmost_bot);
    \end{tikzpicture}
    \hfill
    \begin{tikzpicture}[
        vertex/.style = {draw,circle,inner sep=2pt, minimum height= 5pt, fill opacity=0.2},
        directed/.style = {thick, black!100, >=stealth, ->},
        bi/.style = {thick,black!100, stealth-stealth},
        undirected/.style = {thick,black!100, >=stealth,-}
        ]
        
        \definecolor{blue}{rgb}{0.3,0.3, 0.8}
        \definecolor{green}{rgb}{0.3,0.8, 0.3}
        \definecolor{red}{rgb}{0.9,0.2, 0.2}
        \definecolor{gray3}{rgb}{0.9,0.9,.9}
        \definecolor{gray2}{rgb}{0.3,0.3,0.3}

        \draw[draw=none, fill=purple!50,fill opacity=0.4] (-0.9,0.8) rectangle (-2.0,-0.8); 
        
        \node[] (s) at (-1.55, 0.65) {$s$};
        \node[] (r) at (-1.55, -0.6) {$r$};
        \node[] (u) at (-1.24, 0.3) {$u$};  
        \node[] (v) at (-0.62, 0.3) {$v$};
        \node[] (w) at (0, 0.3) {$w$};
        
        \node[vertex, fill=red] (r) at (-1.8, -0.62) {};
        \node[vertex, fill=gray] (s) at (-1.8, 0.62) {};
        \node[star,star points=7,star point ratio=0.7, fill=yellow] (u) at (-1.24, 0) {};  
         \node[vertex, fill=yellow] (u) at (-1.24, 0) {};
        \node[vertex, fill=orange] (v) at (-0.62, 0) {};
        \node[vertex, fill=black] (w) at (0, 0) {};
        \node[vertex, fill=orange] (w_top_right) at (0.55, 0.46) {};
        \node[vertex, fill=orange] (w_bot_right) at (0.55, -0.46) {};
            \node[vertex, fill=green] (rightmost_top) at (1.1, 0.54) {};
        \node[vertex, fill=blue] (rightmost_bot) at (1.1, -0.54) {};  
        
        \draw[undirected] (s) edge (r); 
        \draw[undirected] (s) edge (u); 
        \draw[undirected] (r) edge (u); 
        \draw[undirected] (u) edge (v);
        \draw[undirected] (v) edge (w); 
        \draw[undirected] (w) edge (w_top_right);
        \draw[undirected] (w) edge (w_bot_right); 
        \draw[undirected] (w_top_right) edge (rightmost_top); 
        \draw[undirected] (w_bot_right) edge (rightmost_bot); 
        \draw[undirected] (rightmost_top) edge (rightmost_bot);
    \end{tikzpicture}
    \vfill
    \vspace{.1cm}
    \begin{tikzpicture}[
        vertex/.style = {draw,circle,inner sep=2pt, minimum height= 5pt},
        directed/.style = {thick, black!100, >=stealth, ->},
        bi/.style = {thick,black!100, stealth-stealth},
        undirected/.style = {thick,black!100, >=stealth,-}
        ]
        
        \definecolor{blue}{rgb}{0.3,0.3, 0.8}
        \definecolor{green}{rgb}{0.3,0.8, 0.3}
        \definecolor{red}{rgb}{0.9,0.2, 0.2}
        \definecolor{gray3}{rgb}{0.9,0.9,.9}
        \definecolor{gray2}{rgb}{0.3,0.3,0.3}

        \draw[draw=none, fill=cyan!20,fill opacity=0.7] (0.25,0.8) rectangle (-0.8,-0.8); 
        
        \node[] (s) at (-1.55, 0.65) {$s$};
        \node[] (r) at (-1.55, -0.6) {$r$};
        \node[] (u) at (-1.24, 0.3) {$u$};  
        \node[] (v) at (-0.62, 0.3) {$v$};
        \node[] (w) at (0, 0.3) {$w$};
        
        \node[vertex, fill=red] (r) at (-1.8, -0.62) {};
        \node[vertex, fill=gray] (s) at (-1.8, 0.62) {};
        \node[vertex, fill=yellow] (u) at (-1.24, 0) {};
        \node[star,star points=7,star point ratio=0.7, fill=orange] (v) at (-0.62, 0) {};  
        \node[vertex] (v) at (-0.62, 0) {};
        \node[vertex, fill=black] (w) at (0, 0) {};
        \node[vertex, fill=orange] (w_top_right) at (0.55, 0.46) {};
        \node[vertex, fill=orange] (w_bot_right) at (0.55, -0.46) {};
            \node[vertex, fill=green] (rightmost_top) at (1.1, 0.54) {};
        \node[vertex, fill=blue] (rightmost_bot) at (1.1, -0.54) {};  
        
        \draw[undirected] (s) edge (r); 
        \draw[undirected] (s) edge (u); 
        \draw[undirected] (r) edge (u); 
        \draw[undirected] (u) edge (v);
        \draw[undirected] (v) edge (w); 
        \draw[undirected] (w) edge (w_top_right);
        \draw[undirected] (w) edge (w_bot_right); 
        \draw[undirected] (w_top_right) edge (rightmost_top); 
        \draw[undirected] (w_bot_right) edge (rightmost_bot); 
        \draw[undirected] (rightmost_top) edge (rightmost_bot);
    \end{tikzpicture}
    \hfill
    \begin{tikzpicture}[
        vertex/.style = {draw,circle,inner sep=2pt, minimum height= 5pt, fill opacity=0.5},
        directed/.style = {thick, black!100, >=stealth, ->},
        bi/.style = {thick,black!100, stealth-stealth},
        undirected/.style = {thick,black!100, >=stealth,-}
        ]
        
        \definecolor{blue}{rgb}{0.3,0.3, 0.8}
        \definecolor{green}{rgb}{0.3,0.8, 0.3}
        \definecolor{red}{rgb}{0.9,0.2, 0.2}
        \definecolor{gray3}{rgb}{0.9,0.9,.9}
        \definecolor{gray2}{rgb}{0.3,0.3,0.3}

        \draw[draw=none, fill=cyan!20,fill opacity=0.7] (0.25,0.8) rectangle (-0.8,-0.8); 
        
        \node[] (s) at (-1.55, 0.65) {$s$};
        \node[] (r) at (-1.55, -0.6) {$r$};
        \node[] (u) at (-1.24, 0.3) {$u$};  
        \node[] (v) at (-0.62, 0.3) {$v$};
        \node[] (w) at (0, 0.3) {$w$};
        
        \node[vertex, fill=red] (r) at (-1.8, -0.62) {};
        \node[vertex, fill=gray] (s) at (-1.8, 0.62) {};
        \node[vertex, fill=yellow] (u) at (-1.24, 0) {};
        \node[star,star points=7,star point ratio=0.7, fill=orange] (v) at (-0.62, 0) {};  
        \node[vertex, fill=orange] (v) at (-0.62, 0) {};
        \node[vertex, fill=black] (w) at (0, 0) {};
        \node[vertex, fill=orange] (w_top_right) at (0.55, 0.46) {};
        \node[vertex, fill=orange] (w_bot_right) at (0.55, -0.46) {};
            \node[vertex, fill=green] (rightmost_top) at (1.1, 0.54) {};
        \node[vertex, fill=blue] (rightmost_bot) at (1.1, -0.54) {};  
        
        \draw[undirected] (s) edge (r); 
        \draw[undirected] (s) edge (u); 
        \draw[undirected] (r) edge (u); 
        \draw[undirected] (u) edge (v);
        \draw[undirected] (v) edge (w); 
        \draw[undirected] (w) edge (w_top_right);
        \draw[undirected] (w) edge (w_bot_right); 
        \draw[undirected] (w_top_right) edge (rightmost_top); 
        \draw[undirected] (w_bot_right) edge (rightmost_bot); 
        \draw[undirected] (rightmost_top) edge (rightmost_bot);
    \end{tikzpicture}
    \hfill
    \begin{tikzpicture}[
        vertex/.style = {draw,circle,inner sep=2pt, minimum height= 5pt, fill opacity=0.2},
        directed/.style = {thick, black!100, >=stealth, ->},
        bi/.style = {thick,black!100, stealth-stealth},
        undirected/.style = {thick,black!100, >=stealth,-}
        ]
        
        \definecolor{blue}{rgb}{0.3,0.3, 0.8}
        \definecolor{green}{rgb}{0.3,0.8, 0.3}
        \definecolor{red}{rgb}{0.9,0.2, 0.2}
        \definecolor{gray3}{rgb}{0.9,0.9,.9}
        \definecolor{gray2}{rgb}{0.3,0.3,0.3}

        \draw[draw=none, fill=cyan!20,fill opacity=0.7] (-0.4,0.8) rectangle (-1.42,-0.8); 
        
        \node[] (s) at (-1.55, 0.65) {$s$};
        \node[] (r) at (-1.55, -0.6) {$r$};
        \node[] (u) at (-1.24, 0.3) {$u$};  
        \node[] (v) at (-0.62, 0.3) {$v$};
        \node[] (w) at (0, 0.3) {$w$};
        
        \node[vertex, fill=red] (r) at (-1.8, -0.62) {};
        \node[vertex, fill=gray] (s) at (-1.8, 0.62) {};
        \node[vertex, fill=yellow] (u) at (-1.24, 0) {};
        \node[star,star points=7,star point ratio=0.7, fill=orange] (v) at (-0.62, 0) {};  
        \node[vertex, fill=orange] (v) at (-0.62, 0) {};
        \node[vertex, fill=black] (w) at (0, 0) {};
        \node[vertex, fill=orange] (w_top_right) at (0.55, 0.46) {};
        \node[vertex, fill=orange] (w_bot_right) at (0.55, -0.46) {};
            \node[vertex, fill=green] (rightmost_top) at (1.1, 0.54) {};
        \node[vertex, fill=blue] (rightmost_bot) at (1.1, -0.54) {};  
        
        \draw[undirected] (s) edge (r); 
        \draw[undirected] (s) edge (u); 
        \draw[undirected] (r) edge (u); 
        \draw[undirected] (u) edge (v);
        \draw[undirected] (v) edge (w); 
        \draw[undirected] (w) edge (w_top_right);
        \draw[undirected] (w) edge (w_bot_right); 
        \draw[undirected] (w_top_right) edge (rightmost_top); 
        \draw[undirected] (w_bot_right) edge (rightmost_bot); 
        \draw[undirected] (rightmost_top) edge (rightmost_bot);
    \end{tikzpicture}
    \caption{Example information flow for nodes $u,v$. \textbf{Top}: information flow relative to $u$ across three layers. Node $u$  listens to every neighbor in the first layer, but only to $v$ in the second layer, and only to $s$ and $r$ in the last layer. \textbf{Bottom}: information flow relative to $v$ across three layers.  The node $v$ listens only to $w$ in the first two layers, and only to  $u$ in the last layer.}
    \label{fig:layerwise-flow}
\end{figure*}
When all nodes perform the action \standardlong, we recover the standard message-passing. Conversely, having all the nodes \isolatelong corresponds to removing all the edges from the graph implying node-wise predictions. The interplay between these actions and the ability to change them {\em dynamically} makes the overall approach richer and allows to decouple the input graph from the computational one and incorporate directionality into message-passing: a node can only listen to those neighbors that are currently broadcasting, and vice versa. 
We can emulate the example from \Cref{fig:layerwise-flow} by making $u$ choose the actions $\langle\listen,\listen,\standard\rangle$, $v$ and $w$ the actions $\langle\standard,\standard,\listen\rangle$, and  $s$ and $r$ the actions $\langle\standard,\isolate,\standard\rangle$.

\textbf{Contributions.}
We develop a new class of architectures, dubbed \emph{cooperative graph neural networks}~(\cognns), where every node in the graph is viewed as a player that can perform one of the aforementioned actions. \cognns comprise two jointly trained ``cooperating'' message-passing neural networks: an \emph{environment network} $\env$ (for solving the given task), and an \emph{action network} $\action$ (for choosing the best actions). 
Our contributions can be summarized as follows:
\begin{enumerate}[noitemsep,topsep=0pt,parsep=1pt,partopsep=0pt,leftmargin=*, label={\small \textbullet}]
    \item We propose a novel message-passing mechanism, which leads to \cognn architectures that effectively explore the graph topology while learning (\Cref{sec:coop}).
    \item  We provide a detailed discussion on the properties of \cognns (\Cref{subsec:conceptual}) and show that they are more expressive than 1-dimensional Weisfeiler-Leman algorithm~(1-WL) (\Cref{subsec:expresivity}), and better suited for long-range tasks due to their adaptive nature (\Cref{subsec:long-range}).
    \item  Empirically, we focus on \cognns with \emph{basic} action and environment networks to carefully assess the virtue of the new message-passing paradigm. We first validate the strength of our approach on a synthetic task (\Cref{subsec:synthetic}). Then, we conduct experiments on real-world datasets, and observe that \cognns always improve compared to their baseline models, and yield multiple state-of-the-art results (\Cref{subsec:node-classification,app:graph-classification}).  
    \item We compare the trend of the actions on homophilic and heterophilic graphs (\Cref{sec:hh}); visualize the actions on a heterophilic graph (\Cref{subsec:visualize}); and ablate on the choices of action and environment networks  (\Cref{subsec:ablation}). We complement these with experiments related to expressive power~(\Cref{app:cycles}), long-range tasks (\Cref{app:long-range-exp}), and over-smoothing (\Cref{app:oversmoothing}).
\end{enumerate}

Additional details can be found in the appendix of this paper.

\section{Background}

\textbf{Graph Neural Networks.} We consider simple, undirected attributed graphs $G=(V, E, \mX)$, where $\mX \in \sR^{|V|\times d}$ is a matrix of (input) node features, and  $\vx_v\in \sR^d$ denotes the feature of a node  $v\in V$. We focus on \emph{message-passing neural networks (MPNNs)}~\citep{GilmerSRVD17} that encapsulate the vast majority of GNNs. An MPNN updates the initial node representations $\vh_{v}^{(0)}=\vx_{v}$ of each node $v$ for $0 \leq \ell \leq L - 1$ iterations based on its own state and the state of its neighbors $\mathcal{N}_v$ as:
\begin{align*}
	\vh_v^\steplplus =\phi^{(\ell)}  \left(
	\vh_v^\stepl, 
	\psi^{(\ell)}\left(\vh_v^\stepl, \{\!\!\{\vh_u^\stepl\mid u\in\mathcal{N}_v\}\!\!\}\right)\right),
\end{align*}
where $\{\!\!\{\cdot\}\!\!\}$ denotes a multiset and $\phi^{(\ell)}$ and $\psi^{(\ell)}$ are differentiable \emph{update} and \emph{aggregation} functions, respectively. We denote by $d^{(\ell)}$ the dimension of the node embeddings at iteration (layer) $\ell$. The final representations $\vh_v^{\left(L\right)}$ of each node $v$ can be used for predicting node-level properties or they can be pooled to form a graph embedding vector $\vz_G^{\left(L\right)}$, which can be used for predicting graph-level properties. The pooling often takes the form of simple averaging, summation, or element-wise maximum. 
Of particular interest to us are the basic MPNNs: 
\begin{equation*}
\vh_v^\steplplus = \sigma\left(\mW^\stepl_{s}\vh_v^\stepl + \mW^\stepl_{n}\psi\left(\{\!\!\{\vh_u^\stepl\mid u\in\mathcal{N}_v\}\!\!\}
\right)\right),
\end{equation*}
where $\mW^\stepl_s$ and  $\mW^\stepl_n$ are
$d^\stepl\times d^\steplplus$ 
learnable parameter matrices acting on the node's self-representation and on the aggregated representation of its neighbors, respectively, $\sigma$ is a non-linearity, and $\psi$ is either \emph{mean} or \emph{sum} aggregation function. We refer to the architecture with mean aggregation as \meangnns and to the architecture with sum aggregation as \sumgnns \citep{hamilton2020graph}. We also consider prominent models such as GCN \citep{Kipf16}, GIN \citep{xu18} and GAT \citep{velic2018graph}.

\textbf{Straight-through Gumbel-softmax Estimator.} In our approach, we rely on an action network for predicting categorical actions for the nodes in the graph, which is not differentiable and poses a challenge for gradient-based optimization. One prominent approach to address this is given by the Gumbel-softmax estimator \citep{jang2017categorical, maddison2017concrete} which effectively provides a differentiable, continuous approximation of discrete action sampling.
Consider a finite set $\Omega$ of actions. We are interested in learning a categorical distribution over $\Omega$, which can be represented in terms of a probability vector $\vp \in \sR^{\lvert \Omega \rvert}$ whose elements store the probabilities of different actions.  Let us denote by $\vp(a)$ the probability of an action $a \in \Omega$. Gumbel-softmax is a special reparametrization trick that estimates the categorical distribution $\vp\in\sR^{\lvert \Omega \rvert}$ with the help of a Gumbel-distributed vector $\vg\in\sR^{\lvert \Omega \rvert}$, which stores an i.i.d.\ sample $\vg(a)\sim\textsc{Gumbel}(0,1)$ for each action $a$.
Given a categorical distribution $\vp$ and a temperature parameter $\tau$, Gumbel-softmax ($\operatorname{GS}$) scores can be computed as follows:
\begin{equation*}
    \operatorname{GS}\left(\vp ;\tau\right)=
    \frac{
        \exp \left( {\left(\log(\vp)+\vg  \right)}/{\tau} \right)
    }
    {
        {\sum_{a\in \Omega} \exp \left( {\left( \log(\vp(a))+\vg(a) \right)} / \tau \right)}
    }
\end{equation*}
As the softmax temperature $\tau$ decreases, the resulting vector tends to a \emph{one-hot} vector. Straight-through GS estimator utilizes the GS estimator during the backward pass only (for a differentiable update), while during the forward pass, it employs an ordinary sampling.

\section{Related Work}
Most of GNNs operate by message-passing \citep{GilmerSRVD17}, including architectures such as  GCNs~\citep{Kipf16}, GIN~\citep{xu18}, GAT~\citep{velic2018graph}, and GraphSAGE~\citep{hamilton2017inductive}. 
Despite their success, MPNNs have some known limitations.

First, the expressive power of MPNNs is upper bounded by 1-WL~\citep{xu18,MorrisAAAI19}. This motivated the study of more expressive architectures, based on higher-order structures~\citep{MorrisAAAI19,MaronBSL19,KerivenP19}, subgraph~\citep{BevilacquaFLSCBBM22,ThiedeZK21} or homomorphism counting~\citep{BarceloGRR21,JBCL-ICML24},  node features with unique identifiers~\citep{Loukas20}, or random features~\citep{AbboudCGL21,SatoSDM2020}.

Second, MPNNs perform poorly on long-range tasks due to their information propagation bottlenecks such as \emph{over-squashing}~\citep{Alon-ICLR21} and \emph{over-smoothing} \cite{LiHW18}.
The former limitation motivated approaches based on rewiring the graph~\citep{KlicperaWG19,topping2022understanding, karhadkar2023fosr} by connecting relevant nodes and shortening propagation distances, or designing new message-passing architectures that act on distant nodes directly, e.g., using shortest-path distances \citep{AbboudDC22,YingCLZKHSL21}. The over-smoothing problem has also motivated a body of work to avoid the collapse of node features~\cite{zhao2019pairnorm,chen2020simple}.

Finally, classical message passing updates the nodes in a fixed and synchronous manner, which does not allow the nodes to react to messages from their neighbors individually. This has been recently argued as yet another limitation of classical message passing from the perspective of algorithmic alignment \citep{faber2022asynchronous}.

Our approach presents new perspectives on these limitations via a dynamic and asynchronous information flow (see \Cref{sec:properties}). This is related to the work of \citet{lai2020policygnn}, where the goal is to update each node using a different number of layers (over a fixed topology). This is also related to the work of \cite{dai2022towards}, where the idea is to apply message passing on a learned topology but one that is the same at every layer. \cognns diverge from these studies in terms of the objectives and the approach.

\begin{figure*}[!ht]
    \begin{subfigure}{0.22\textwidth}
        \centering
	\begin{tikzpicture}[
    	vertex/.style = {draw,circle,inner sep=2pt, minimum height= 5pt},
            directed/.style = {thick, black!100, >=stealth, ->},
    	  bi/.style = {thick,black!100, stealth-stealth},
            undirected/.style = {thick,black!100, >=stealth,-}
           ]

    	\definecolor{blue}{rgb}{0.3,0.3, 0.8}
            \definecolor{green}{rgb}{0.3,0.8, 0.3}
            \definecolor{red}{rgb}{0.9,0.2, 0.2}
            \definecolor{gray3}{rgb}{0.9,0.9,.9}
    	\definecolor{gray2}{rgb}{0.3,0.3,0.3}

            \node[] (l) at (-0.4, 1.1) {$H$};
            
            \node[] (s) at (-1.55, 0.65) {$s$};
            \node[] (r) at (-1.55, -0.6) {$r$};
            \node[] (u) at (-1.24, 0.3) {$u$};  
            \node[] (v) at (-0.62, 0.3) {$v$};
            \node[] (w) at (0, 0.3) {$w$};

    	\node[vertex, fill=red] (r) at (-1.8, -0.62) {};
    	\node[vertex, fill=gray] (s) at (-1.8, 0.62) {};
    	\node[vertex, fill=yellow] (u) at (-1.24, 0) {};
    	\node[vertex, fill=orange] (v) at (-0.62, 0) {};
    	\node[vertex, fill=black] (w) at (0, 0) {};
    	\node[vertex, fill=orange] (w_top_right) at (0.55, 0.46) {};
    	\node[vertex, fill=orange] (w_bot_right) at (0.55, -0.46) {};
            \node[vertex, fill=green] (rightmost_top) at (1.1, 0.54) {};
    	\node[vertex, fill=blue] (rightmost_bot) at (1.1, -0.54) {};  

    	\draw[undirected] (s) edge (r); 
    	\draw[undirected] (s) edge (u); 
    	\draw[undirected] (r) edge (u); 
    	\draw[undirected] (u) edge (v);
    	\draw[undirected] (v) edge (w); 
    	\draw[undirected] (w) edge (w_top_right);
    	\draw[undirected] (w) edge (w_bot_right); 
    	\draw[undirected] (w_top_right) edge (rightmost_top); 
    	\draw[undirected] (w_bot_right) edge (rightmost_bot); 
    	\draw[undirected] (rightmost_top) edge (rightmost_bot);
	\end{tikzpicture}
        \caption{Input graph $H$.}
    \end{subfigure}
    \hfill
        \begin{subfigure}{0.23\textwidth}
        \centering
	\begin{tikzpicture}[
    	vertex/.style = {draw,circle,inner sep=2pt, minimum height= 5pt},
            directed/.style = {thick, black!100, >=stealth, ->},
    	  bi/.style = {thick,black!100, stealth-stealth},
            undirected/.style = {thick,black!100, >=stealth,-}
           ]

    	\definecolor{blue}{rgb}{0.3,0.3, 0.8}
            \definecolor{green}{rgb}{0.3,0.8, 0.3}
            \definecolor{red}{rgb}{0.9,0.2, 0.2}
            \definecolor{gray3}{rgb}{0.9,0.9,.9}
    	\definecolor{gray2}{rgb}{0.3,0.3,0.3}

            \node[] (l) at (-0.4, 1.1) {$H^{(0)}$};
            
            \node[] (s) at (-1.55, 0.65) {$s$};
            \node[] (r) at (-1.55, -0.6) {$r$};
            \node[] (u) at (-1.24, 0.3) {$u$};  
            \node[] (v) at (-0.62, 0.3) {$v$};
            \node[] (w) at (0, 0.3) {$w$};

    	\node[vertex, fill=red] (r) at (-1.8, -0.62) {};
    	\node[vertex, fill=gray] (s) at (-1.8, 0.62) {};
    	\node[vertex, fill=yellow] (u) at (-1.24, 0) {};
    	\node[vertex, fill=orange] (v) at (-0.62, 0) {};
    	\node[vertex, fill=black] (w) at (0, 0) {};
    	\node[vertex, fill=orange] (w_top_right) at (0.55, 0.46) {};
    	\node[vertex, fill=orange] (w_bot_right) at (0.55, -0.46) {};
            \node[vertex, fill=green] (rightmost_top) at (1.1, 0.54) {};
    	\node[vertex, fill=blue] (rightmost_bot) at (1.1, -0.54) {};  

    	\draw[bi] (s) edge (r); 
    	\draw[directed] (s) edge (u); 
    	\draw[directed] (r) edge (u); 
    	\draw[directed] (v) edge (u);
    	\draw[bi] (v) edge (w); 
    	\draw[bi] (w) edge (w_top_right);
    	\draw[bi] (w) edge (w_bot_right); 
    	\draw[bi] (w_top_right) edge (rightmost_top); 
    	\draw[bi] (w_bot_right) edge (rightmost_bot); 
    	\draw[bi] (rightmost_top) edge (rightmost_bot);
	\end{tikzpicture}
        \caption{Computation at $\ell=0$.}
    \end{subfigure}
    \hfill
    \begin{subfigure}{0.23\textwidth}
        \centering
	\begin{tikzpicture}[
    	vertex/.style = {draw,circle,inner sep=2pt, minimum height= 5pt, fill opacity=0.5},
            directed/.style = {thick, black!100, >=stealth, ->},
    	  bi/.style = {thick,black!100, stealth-stealth},
            undirected/.style = {thick,black!100, >=stealth,-}
           ]

    	\definecolor{blue}{rgb}{0.3,0.3, 0.8}
            \definecolor{green}{rgb}{0.3,0.8, 0.3}
            \definecolor{red}{rgb}{0.9,0.2, 0.2}
            \definecolor{gray3}{rgb}{0.9,0.9,.9}
    	\definecolor{gray2}{rgb}{0.3,0.3,0.3}

            \node[] (l) at (-0.4, 1.1) {$H^{(1)}$};
            
            \node[] (s) at (-1.55, 0.65) {$s$};
            \node[] (r) at (-1.55, -0.6) {$r$};
            \node[] (u) at (-1.24, 0.3) {$u$};  
            \node[] (v) at (-0.62, 0.3) {$v$};
            \node[] (w) at (0, 0.3) {$w$};

    	\node[vertex, fill=red] (r) at (-1.8, -0.62) {};
    	\node[vertex, fill=gray] (s) at (-1.8, 0.62) {};
    	\node[vertex, fill=yellow] (u) at (-1.24, 0) {};
    	\node[vertex, fill=orange] (v) at (-0.62, 0) {};
    	\node[vertex, fill=black] (w) at (0, 0) {};
    	\node[vertex, fill=orange] (w_top_right) at (0.55, 0.46) {};
    	\node[vertex, fill=orange] (w_bot_right) at (0.55, -0.46) {};
            \node[vertex, fill=green] (rightmost_top) at (1.1, 0.54) {};
    	\node[vertex, fill=blue] (rightmost_bot) at (1.1, -0.54) {};  

    	\draw[directed] (v) edge (u);
    	\draw[bi] (v) edge (w); 
    	\draw[bi] (w) edge (w_top_right);
    	\draw[bi] (w) edge (w_bot_right); 
    	\draw[bi] (w_top_right) edge (rightmost_top); 
    	\draw[bi] (w_bot_right) edge (rightmost_bot); 
    	\draw[bi] (rightmost_top) edge (rightmost_bot);
	\end{tikzpicture}
        \caption{Computation at $\ell=1$.}
    \end{subfigure}
    \hfill
    \begin{subfigure}{0.23\textwidth}
        \centering
        \begin{tikzpicture}[
    	vertex/.style = {draw,circle,inner sep=2pt, minimum height= 5pt, fill opacity=0.2},
            directed/.style = {thick, black!100, >=stealth, ->},
    	  bi/.style = {thick,black!100, stealth-stealth},
            undirected/.style = {thick,black!100, >=stealth,-}
           ]

    	\definecolor{blue}{rgb}{0.3,0.3, 0.8}
            \definecolor{green}{rgb}{0.3,0.8, 0.3}
            \definecolor{red}{rgb}{0.9,0.2, 0.2}
            \definecolor{gray3}{rgb}{0.9,0.9,.9}
    	\definecolor{gray2}{rgb}{0.3,0.3,0.3}

            \node[] (l) at (-0.4, 1.1) {$H^{(2)}$};
            
            \node[] (s) at (-1.55, 0.65) {$s$};
            \node[] (r) at (-1.55, -0.6) {$r$};
            \node[] (u) at (-1.24, 0.3) {$u$};  
            \node[] (v) at (-0.62, 0.3) {$v$};
            \node[] (w) at (0, 0.3) {$w$};

    	\node[vertex, fill=red] (r) at (-1.8, -0.62) {};
    	\node[vertex, fill=gray] (s) at (-1.8, 0.62) {};
    	\node[vertex, fill=yellow] (u) at (-1.24, 0) {};
    	\node[vertex, fill=orange] (v) at (-0.62, 0) {};
    	\node[vertex, fill=black] (w) at (0, 0) {};
    	\node[vertex, fill=orange] (w_top_right) at (0.55, 0.46) {};
    	\node[vertex, fill=orange] (w_bot_right) at (0.55, -0.46) {};
            \node[vertex, fill=green] (rightmost_top) at (1.1, 0.54) {};
    	\node[vertex, fill=blue] (rightmost_bot) at (1.1, -0.54) {};  

    	\draw[bi] (s) edge (r); 
    	\draw[bi] (s) edge (u); 
    	\draw[bi] (r) edge (u); 
    	\draw[directed] (u) edge (v);
    	\draw[directed] (w_top_right) edge (w);
    	\draw[directed] (w_bot_right) edge (w); 
    	\draw[bi] (w_top_right) edge (rightmost_top); 
    	\draw[bi] (w_bot_right) edge (rightmost_bot); 
    	\draw[bi] (rightmost_top) edge (rightmost_bot);
	\end{tikzpicture}
        \caption{Computation at $\ell=2$.}
    \end{subfigure}
    \hfill
    \caption{ The input graph $H$ and its computation graphs $H^{(0)}$, $H^{(1)}$, $H^{(2)}$ that are a result of applying the actions:  $\langle\listen,\listen,\standard\rangle$ for the node $u$; $\langle\standard,\standard,\listen\rangle$ for the nodes $v$ and $w$;  $\langle\standard,\isolate,\standard\rangle$ for the nodes $s$ and $r$;  $\langle\standard,\standard,\standard\rangle$ for all other nodes.}
    \label{fig:computational_graph}
\end{figure*}
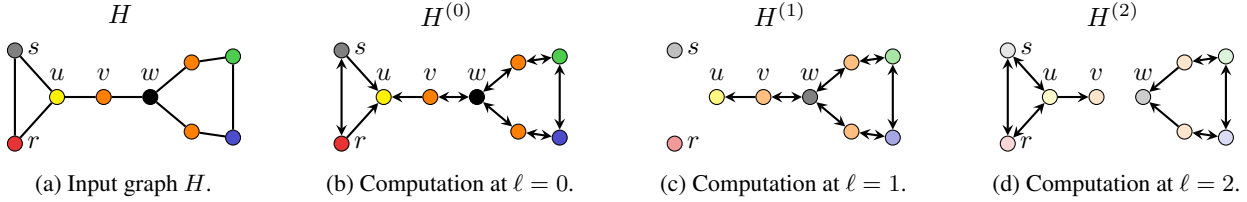

\section{Cooperative Graph Neural Networks}
\label{sec:coop}
$\cognns$ view each node in a graph as a \emph{player} of a multiplayer environment, where the state of each player is given in terms of the representation (or {\em state}) of its corresponding node. Every node is updated following a two-stage process. In the first stage, each node chooses an action from the set of actions given their current state and the states of their neighboring nodes. In the second stage, every node state gets updated based on their current state and the states of a \emph{subset} of the neighboring nodes, as determined by the actions in the first stage. As a result, every node can determine how to propagate information from or to its neighbors.

A $\cognn\left(\action, \env\right)$ architecture is given in terms of two cooperating GNNs: (i) an action network $\action$ for choosing the best actions, and (ii) an environment network $\env$ for updating the node representations. A $\cognn$ layer updates the representations $\vh_v^{(\ell)}$ of each node $v$ as follows. First, an action network $\action$ predicts, for each node $v$, a probability distribution $\vp_v^{\left(\ell \right)}\in\sR^4$ over the actions $\{\standard, \listen,\broadcast, \isolate\}$ that $v$ can take, given its state and the state of its neighbors $\mathcal{N}_v$:
\begin{align}
\label{label:action}
        \vp_v^{(\ell)}= \pi \left(\vh_v^{(\ell)}, \{\!\!\{\vh_u^{(\ell)}\mid u\in\mathcal{N}_v\}\!\!\} \right). 
\end{align}
Then, for each node $v$, an action is sampled $a_v^{\left(\ell \right)} \sim \vp_v^{\left(\ell \right)}$ using Straight-through GS, and an environment network $\env$ is utilized to update the state of each node in accordance with the sampled actions:
\begin{align}
\label{label:env}
        \vh_v^\steplplus = 
        \begin{cases}
            \eta^\stepl \big( \vh_v^\stepl, \{\!\!\{ \}\!\!\}\big), & a_v^{\left(\ell \right)} = \isolate \lor \broadcast   \\
            \eta^\stepl\big( \vh_v^\stepl, \mathcal{M} \big), & a_v^{\left(\ell \right)} = \listen \lor \standard
        \end{cases}
\end{align}
where $\mathcal{M}=\{\!\!\{  \vh_u^\stepl \mid u \in \mathcal{N}_v, a_u^{\left(\ell \right)} =\standard \lor \broadcast \}\!\!\}$.

This is a single layer update, and by stacking $L\geq 1$ layers, we obtain the representations $\vh_v^{\left(L\right)}$ for each node $v$.

In its full generality, a $\cognn\left(\action, \env\right)$ architecture can operate on (un)directed graphs and use any GNN architecture in place of the action network $\action$ and the environment network $\env$. To carefully assess the virtue of this new message-passing paradigm, we consider relatively simple architectures such as $\sumgnns$, $\meangnns$, GCN, GIN and GAT, which are respectively denoted as $\sum$, $\mu$, $*$, $\epsilon$, and $\alpha$. For example, we write $\comeansum$ to denote a \cognn architecture which uses $\sumgnn$ as its action network and $\meangnn$ as its environment network.

Fundamentally, \cognns update the node states in a fine-grained manner: if a node $v$ chooses to \isolatelong or to \broadcastlong then it gets updated only based on its previous state, which corresponds to a node-wise update function. On the other hand, if a node $v$ chooses the action \listenlong or \standardlong then it gets updated based on its previous state as well as the state of its neighbors which perform the actions \broadcastlong or \standardlong at this layer. 

\section{Model Properties}
\label{sec:properties}

We analyze \cognns, focusing on conceptual novelty, expressive power, and suitability to long-range tasks.

\subsection{Conceptual Properties}
\label{subsec:conceptual}
\textbf{Task-specific}: Standard message-passing updates nodes based on their neighbors, which is completely task-agnostic. By allowing each node to listen to the information from `relevant' neighbors only, \cognns can determine a computation graph which is best suited for the target task. For example, if the task requires information only from the neighbors with a certain degree then the action network can learn to listen only to these nodes (see~\Cref{subsec:synthetic}). 

\textbf{Directed}: The outcome of the actions that the nodes can take amounts to a special form of \emph{`directed rewiring'} of the input graph: an edge can be {\em dropped} (e.g., if two neighbors listen without broadcasting); an edge can remain {\em undirected} (e.g., if both neighbors apply the standard action); or, an edge can {\em become directed} implying directional information flow (e.g., if one neighbor listens while its neighbor broadcasts). 
Taking this perspective, the proposed message-passing can be seen as operating on a potentially different directed graph induced by the choice of actions at every layer (illustrated in \Cref{fig:computational_graph}). 
Formally, given a graph $G=(V,E)$, let us denote by $G^\stepl=(V,E^\stepl)$ the directed computational graphs induced by the actions chosen at layer $\ell$, where $E^\stepl$ is the set of directed edges at layer $\ell$. We can rewrite the update given in \Cref{label:env} concisely as follows:
\begin{equation*}
    \vh_v^\steplplus = 
        \env^\stepl\left(
            \vh_v^\stepl,
            \{\!\!\{\vh_u^\stepl\mid (u,v)\in E^\stepl\}\!\!\}
        \right).
\end{equation*}
Consider the input graph $H$ from \Cref{fig:computational_graph}: $u$ gets messages from $v$ only in the first two layers, and $v$ gets messages from $u$ only in the last layer, illustrating a directional message-passing between these nodes. This abstraction allows for a direct implementation of \cognns by simply considering the induced graph adjacency matrix at every layer.

\begin{wrapfigure}{r}{0.2\textwidth} 
\scalebox{0.85}{
    \begin{tikzpicture}[
        vertex/.style = {draw,circle,inner sep=2pt, minimum height= 5pt},
            black_edge/.style = {->, shorten >=1pt,shorten <=2pt, thick, >=stealth},
            gray_edge/.style = {->, shorten >=1pt,shorten <=2pt, thick, >=stealth, gray!90},
            red_edge/.style = {->, shorten >=1pt,shorten <=2pt, thick, >=stealth, line width=1.5pt, red}
           ]
    
        \definecolor{blue}{rgb}{0.3,0.3, 0.8}
        \definecolor{green}{rgb}{0.3,0.8, 0.3}
        \definecolor{red}{rgb}{0.9,0.2, 0.2}
        \definecolor{gray3}{rgb}{0.9,0.9,.9}
        \definecolor{gray2}{rgb}{0.3,0.3,0.3}  

        \node[] (v) at (0.0, 0.25) {$v$};
        \node[] (u) at (0.8, 0.75) {$u$};
        \node[] (w) at (0.8, -0.25) {$w$};

        \node[] (s) at (1.6,1.25) {$s$};
        \node[] (r) at (1.6,0.75) {$r$};
        \node[] (v2) at (1.6,0.25) {$v$};

        \node[] (r2) at (2.7,1.5) {$r$};
        \node[] (u2) at (2.7,0.9) {$u$};
        \node[] (s2) at (2.7,0.3) {$s$};
        \node[] (w2) at (2.7,-0.9) {$w$};

        \node[vertex, fill=orange] (v) at (0,0) {\quad};
        
        \node[vertex, fill=yellow] (u) at (0.8,0.5) {\quad};
        \node[vertex, fill=black] (w) at (0.8,-0.5) {\quad};

        \node[vertex, fill=gray] (s) at (1.6,1.0) {\quad};
        \node[vertex, fill=red] (r) at (1.6,0.5) {\quad};
        \node[vertex, fill=orange] (v2) at (1.6,0.0) {\quad};
        \node[vertex, fill=orange] (w_top_right) at (1.6,-0.5) {\quad};
        \node[vertex, fill=orange] (w_bot_right) at (1.6,-1.0) {\quad};

        \node[vertex, fill=red] (r2) at (2.4,1.5) {\quad};
        \node[vertex, fill=yellow] (u2) at (2.4,0.9) {\quad};
        \node[vertex, fill=gray] (s2) at (2.4,0.3) {\quad};
        \node[vertex, fill=black] (w2) at (2.4,-0.9) {\quad};
        \node[vertex, fill=green] (rightmost_top) at (2.4, -0.3) {};
        \node[vertex, fill=blue] (rightmost_bot) at (2.4, -1.5) {}; 

        \draw[red_edge] (u) to (v);
        \draw[gray_edge] (w) to (v);
        
        \draw[gray_edge] (r) to (u);
        \draw[gray_edge] (s) to (u);
        \draw[red_edge] (v2) to (u);
        \draw[black_edge] (v2) to (w);
        \draw[black_edge] (w_top_right) to (w);
        \draw[black_edge] (w_bot_right) to (w);

        \draw[black_edge] (r2) to (s);
        \draw[gray_edge] (u2) to (s);
        \draw[gray_edge] (u2) to (r);
        \draw[black_edge] (s2) to (r);
        \draw[gray_edge] (u2) to (v2);
        \draw[red_edge] (w2) to (v2);
        \draw[black_edge] (rightmost_top) to (w_top_right);
        \draw[black_edge] (w2) to (w_top_right);
        \draw[black_edge] (w2) to (w_bot_right);
        \draw[black_edge] (rightmost_bot) to (w_bot_right);

        \draw[dotted] (-0.4,2.0) -- (-0.4,-2.0);
        \draw[dotted] (0.4,2.0) -- (0.4,-2.0);
        \draw[dotted] (1.2,2.0) -- (1.2,-2.0);
        \draw[dotted] (2.0,2.0) -- (2.0,-2.0);

        \node[draw=none] (t_3) at (-0.4,-2.1) {\small $\ell=$};
        \node[draw=none] (t_3) at (0.0,-2.1) {\small $3$};
        \node[draw=none] (t_2) at (0.8,-2.1) {\small $2$};
        \node[draw=none] (t_1) at (1.6,-2.1) {\small $1$};
        \node[draw=none] (t_0) at (2.4,-2.1) {\small $0$};
    \end{tikzpicture}
    }
\end{wrapfigure}    
\textbf{Dynamic}: In \cognns, each node interacts with the `relevant' neighbors and does so only as long as they remain relevant. \cognns do not operate on a fixed computational graph, but rather on a learned computational graph, which is dynamic across layers. In our running example, the computational graph is a different one at every layer (depicted on the right hand side):  This is advantageous for the information flow (see~\Cref{subsec:long-range}). 

\textbf{Feature and Structure Based}: Standard message-passing is determined by the structure of the graph: two nodes with the same neighborhood get the same aggregated message. This is not necessarily the case in our setup, since the action network can learn different actions for two nodes with different node features, e.g., by choosing different actions for a \emph{red} node and a \emph{blue} node. This enables different messages for different nodes even if their neighborhoods are identical.

\textbf{Asynchronous}: Standard message-passing updates all nodes synchronously, which is not always optimal as argued by \citet{faber2022asynchronous}, especially when the task requires to treat the nodes non-uniformly.  By design, \cognns enable asynchronous updates across nodes.

\textbf{Conditional Aggregation}: The action network of \cognns can be viewed as a look-ahead function that makes decisions after applying $k$ layers. Specifically, at layer $\ell$, an action network of depth $k$ computes node representations on the original graph topology, which are $(k+\ell)$-layer representations. Based on these representations, the action network determines an action for each node, which induces a new graph topology for the environment network to operate on. In this sense, the aggregation of environment network at layer $\ell$ is determined by $(k+\ell)$-layer representations of the action network, which can be viewed as a ``look-ahead'' capability and the aggregation mechanism of the environment network is conditioned on this look-ahead capability. 

\textbf{Orthogonal to Attention}: The (soft) attention mechanism on graphs allows for aggregating --- based on learnable attention coefficients ---  a weighted mean of the features of neighboring nodes. While these architectures can weigh the contribution of different neighbors, they have certain limitations, e.g., weighted mean aggregation cannot count node degrees. Moreover, the conditional aggregation mechanism of \cognns goes beyond the capabilities of attention-based architectures.  The contribution of Co-GNNs is hence orthogonal to that of attention-based architectures, such as GATs, and these architectures can be used as base action/environment architectures in \cognns. In \Cref{subsec:synthetic}, we empirically validate this via a task that GAT cannot solve, but \cognns with a GAT environment network can.

\textbf{Mitigates Over-smoothing}: In principle, the action network of \cognns can choose the action $\isolatelong$ for a node if the features of the neighbours of this node are not informative. As a result, \cognns can mitigate over-smoothing. We validate this empirically in \Cref{app:oversmoothing}, but we also note that the optimisation becomes increasingly difficult once the number of layers gets too large.

\textbf{Efficient}: While being more sophisticated, our approach is efficient in terms of runtime, as we detail in \Cref{app:runtime}. \cognns are also parameter-efficient: they share the same action network across layers and as a result a comparable number of parameters to their baseline models.

\subsection{Expressive Power of \cognns}
\label{subsec:expresivity}
The environment and action networks of $\cognn$ architectures are parameterized by standard MPNNs. This raises an obvious question regarding the expressive power of $\cognn$ architectures: {\em are $\cognns$ also bounded by 1-WL in terms of distingushing graphs?} 

\begin{proposition}
\label{prop:expresivity}
  Let $G_1=(V_1,E_1,\mX_1)$ and $G_2=(V_2,E_2,\mX_2)$ be two non-isomorphic graphs. Then, for any threshold $0 < \delta < 1$, there exists a parametrization of a $\cognn$ architecture using sufficiently many layers $L$, satisfying $\mathbb{P}(\vz_{G_1}^{\left(L\right)} \neq \vz_{G_2}^{\left(L\right)}) \ge 1 - \delta$.
\end{proposition}

The explanation for this result is the following: $\cognn$ architectures learn, at every layer, and for each node $u$, a probability distribution over the actions. These learned distributions are identical for two isomorphic nodes. However, the process relies on \emph{sampling} actions from these distributions, and clearly, the samples from identical distributions can differ. This makes $\cognn$ models {\em invariant in expectation,} and the variance introduced by the sampling process helps to discriminate nodes that are 1-WL indistinguishable. Thus, for two nodes indistinguishable by 1-WL, there is a non-trivial probability of sampling a different action for the respective nodes, which in turn makes their direct neighborhood differ. This yields unique node identifiers (\citet{Loukas20}) with high probability and allows us to distinguish any pair of graphs assuming an injective graph pooling function \citep{xu18}.
This is analogous to GNNs with random node features~\citep{AbboudCGL21,SatoSDM2020}, which are more expressive than their classical counterparts. We validate the stated expressiveness gain in \Cref{app:cycles}. 

It is important to note that \cognns are \emph{not} designed for expressiveness, and our result relies merely on variations in the sampling process, which is unstable and should be noted as a limitation. Clearly, \cognns can also use more expressive architectures as base architectures. 

\subsection{Dynamic Message-passing for Long-range Tasks}
\label{subsec:long-range}
Long-range tasks necessitate to propagate information between distant nodes: \cognns are effective for such tasks since they can propagate only relevant task-specific information.
Suppose that we are interested in transmitting information from a source node to a distant target node: \cognns can efficiently filter irrelevant information by learning to focus on a path connecting these two nodes, hence maximizing the information flow to the target node. We can generalize this observation towards receiving information from multiple distant nodes and prove the following:

\begin{proposition}\label{prop:long-range}
Let $G=(V,E,\mX)$ be a connected graph with node features. For some $k>0$, for any target node $v \in V$, for any $k$ source nodes $u_1, \ldots, u_k \in V$, and for any compact, differentiable function $f: \mathbb{R}^{d^{(0)}} \times \ldots \times \mathbb{R}^{d^{(0)}} \to \mathbb{R}^d$, there exists an $L$-layer $\cognn$ computing final node representations such that for any $\epsilon, \delta >0$ it holds that $\mathbb{P} (|\vh_v^{(L)} - f(\vx_{u_1}, \ldots \vx_{u_k})| < \epsilon) \geq 1-\delta$.
\end{proposition}

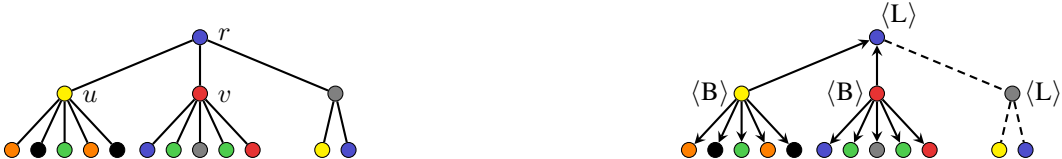
\begin{figure*}[ht!]
    \centering
	\begin{tikzpicture}[
            level distance = 0.75cm,
            level 1/.style={sibling distance=1.8cm},
            level 2/.style={sibling distance=0.35cm},
            every node/.style = {draw,circle,inner sep=2pt, minimum height= 5pt},
            noedge/.style = {draw=none},
            edge from parent/.style = {draw=none},
            ] 
            \definecolor{blue}{rgb}{0.3,0.3, 0.8}
            \definecolor{green}{rgb}{0.3,0.8, 0.3}
            \definecolor{red}{rgb}{0.9,0.2, 0.2}
            \definecolor{gray3}{rgb}{0.9,0.9,.9}
            \definecolor{gray2}{rgb}{0.3,0.3,0.3}
            \begin{scope}
                \node [fill=blue] (r){\quad}
                child { node [fill=yellow](r1) {\quad}
                child { node [fill=orange]   (r11) {\quad} }
                child { node [fill=black] (r12) {\quad} }
                child { node [fill=green] (r13) {\quad} }
                child { node [fill=orange] (r14) {\quad} }
                child { node [fill=black] (r15) {\quad} }
                }
                child  { node [fill=red](r2) {\quad}
                child { node [fill=blue]  (r21) {\quad} }
                child { node [fill=green] (r22) {\quad} }
                child { node [fill=gray]   (r23) {\quad} }
                child { node [fill=green] (r24) {\quad} }
                child { node [fill=red]   (r25) {\quad} }
                }
                child { node [fill=gray](r3) {\quad}
                child { node [fill=yellow] (r31) {\quad} }
                child { node [fill=blue] (r32) {\quad} }
                };
    
                \draw[thick,>=stealth] (r1) to (r);
                \draw[thick,>=stealth] (r2) to (r);
                \draw[thick,>=stealth] (r3) to (r);
    
                \draw[thick,>=stealth] (r1) edge (r11);
                \draw[thick,>=stealth] (r1) edge (r12);
                \draw[thick,>=stealth] (r1) edge (r13);
                \draw[thick,>=stealth] (r1) edge (r14);
                \draw[thick,>=stealth] (r1) edge (r15);
                \draw[thick,>=stealth] (r2) edge (r21);
                \draw[thick,>=stealth] (r2) edge (r22);
                \draw[thick,>=stealth] (r2) edge (r23);
                \draw[thick,>=stealth] (r2) edge (r24);
                \draw[thick,>=stealth] (r2) edge (r25);
                \draw[thick,>=stealth] (r3) edge (r31);
                \draw[thick,>=stealth] (r3) edge (r32);
    
                \node[right, xshift = 0.3em, yshift = +0.1em, draw=none] at (r) {$r$};
                \node[right, xshift = 0.3em, yshift = -0.1em, draw=none] at (r1) {$u$};
                \node[right, xshift = 0.3em, yshift = -0.1em, draw=none] at (r2) {$v$};
            \end{scope}
            \begin{scope}[xshift=9cm]
                \node [fill=blue] (r){\quad}
                child { node [fill=yellow](r1) {\quad}
                child { node [fill=orange]   (r11) {\quad} }
                child { node [fill=black] (r12) {\quad} }
                child { node [fill=green] (r13) {\quad} }
                child { node [fill=orange] (r14) {\quad} }
                child { node [fill=black] (r15) {\quad} }
                }
                child  { node [fill=red](r2) {\quad}
                child { node [fill=blue]  (r21) {\quad} }
                child { node [fill=green] (r22) {\quad} }
                child { node [fill=gray]   (r23) {\quad} }
                child { node [fill=green] (r24) {\quad} }
                child { node [fill=red]   (r25) {\quad} }
                }
                child { node [fill=gray](r3) {\quad}
                child { node [fill=yellow] (r31) {\quad} }
                child { node [fill=blue] (r32) {\quad} }
                };
    
                \draw[->, thick,>=stealth] (r1) to (r);
                \draw[->, thick,>=stealth] (r2) to (r);
                
                \draw[thick, densely dashed] (r) edge (r3);
    
                \draw[->, thick,>=stealth] (r1) edge (r11);
                \draw[->, thick,>=stealth] (r1) edge (r12);
                \draw[->, thick,>=stealth] (r1) edge (r13);
                \draw[->, thick,>=stealth] (r1) edge (r14);
                \draw[->, thick,>=stealth] (r1) edge (r15);
                \draw[->, thick,>=stealth] (r2) edge (r21);
                \draw[->, thick,>=stealth] (r2) edge (r22);
                \draw[->, thick,>=stealth] (r2) edge (r23);
                \draw[->, thick,>=stealth] (r2) edge (r24);
                \draw[->, thick,>=stealth] (r2) edge (r25);
                \draw[thick, densely dashed] (r3) edge (r31);
                \draw[thick, densely dashed] (r3) edge (r32);
                
                \node[above right,draw=none] at (r) {$\langle\listen\rangle$};
                \node[left,draw=none] at (r1) {$\langle\broadcast\rangle$};
                \node[left,draw=none] at (r2) {$\langle\broadcast\rangle$};
                \node[right,draw=none] at (r3) {$\langle\listen\rangle$}; 
            \end{scope}
        \end{tikzpicture}
    \caption{\rootn examples. \textbf{Left:} Example tree for \rootn.
   \textbf{Right:} Example of an optimal directed subgraph over the input tree, where the nodes with a degree of 6 ($u$ and $v$) $\broadcastlong$, while other nodes $\listenlong$.
    }
    \label{fig:combined_synthetic}
\end{figure*}

This means that if a property of a node $v$ is a function of $k$ distant nodes then \cognns can approximate this function. This follows from two findings: (i) the features of $k$ nodes can be transmitted to the source node without loss of information and (ii) the final layer of a \cognn architecture, e.g., an MLP, can approximate any differentiable function over $k$ node features~\citep{Hornik91,Cybenko89}. We validate these findings empirically on long-range interactions datasets \citep{dwivedi2023long} in \Cref{app:long-range-exp}.

\section{Experimental Results}
\label{sec:exp}
We evaluate $\cognns$ on a synthetic experiment,  and on real-world node classification datasets  \citep{platonov2023critical}. We also report a synthetic expressiveness experiment, an experiment on long-range interactions datasets \citep{dwivedi2023long}, and graph classification datasets \citep{Morris2020} in \Cref{app:additional-experiments}. Our codebase is available at \url{https://github.com/benfinkelshtein/CoGNN}.
\subsection{Synthetic Experiment on \rootn}
\label{subsec:synthetic}

\textbf{Task.} In this experiment, we compare \cognns to MPNNs on a new dataset: \rootn. We consider the following regression task: \emph{given a rooted tree, predict the average of the features of root-neighbors of degree $6$.} This task requires to first identify the neighbors of the root node with degree $6$ and then to return the average feature of these nodes. \rootn consists of trees of depth $2$ with random features of dimension 5. The generation (\Cref{app:synthetic_generation}) ensures each tree root has at least one degree-$6$ neighbor. An example is shown on the left of \Cref{fig:combined_synthetic}: the root node $r$ has only two neighbors with degree $6$ ($u$ and $v$) and the target prediction value is $(\vx_u+\vx_v)/2$.

\textbf{Setup.} We consider GCN, GAT, $\sumgnn$, $\meangnn$, as baselines,  and compare to $\cosum$, $\comean$, $\cogatsum$ and $\comeansum$. We report the Mean Average Error (MAE), use the Adam optimizer and present all details including the hyperparameters in \Cref{app:hyperparameters}.

\textbf{Results for MPNNs.} The results are presented in \Cref{tab:synthetic}, which includes the random baseline (i.e., MAE obtained via a random prediction). All MPNNs perform poorly: GCN, GAT, and $\meangnn$ fail to identify node degrees, making it impossible to detect nodes with a specific degree, which is crucial for the task. GCN and GAT are only marginally better than the random baseline, whereas $\meangnn$ performs substantially better than the random baseline. The latter can be explained by the fact that \meangnn employs a different transformation on the source node rather than treating it as a neighbor (unlike the self-loop in GCN/GAT). $\sumgnn$ uses sum aggregation and can identify the node degrees, but struggles in averaging the node features, which yields comparable MAE results to that of \meangnn.

\textbf{Results for \cognns.} The ideal mode of operation for \cognns would be as follows: 
\begin{enumerate}[noitemsep,topsep=0pt,parsep=1pt,partopsep=0pt,leftmargin=*]
    \item The action network chooses either \listenlong or \standardlong for the root node, and \broadcastlong  or \standardlong for the root-neighbors which have a degree $6$.
    \item The action network chooses either \listenlong or \isolatelong for all the remaining root-neighbors.
    \item The environment network updates the root node by averaging features from its broadcasting neighbors.
\end{enumerate}

\begin{table}[t]
\caption{Results on \rootn. Top three models are
colored by \red{First}, \blue{Second}, \gray{Third}.}
\label{tab:synthetic}
  \centering
  \begin{tabular}{lc}
    \toprule
    Model & MAE\\
    \midrule
    Random             & 0.474\\
    GAT                & 0.442\\
    $\sumgnn$          & 0.370\\
    $\meangnn$         & 0.329\\
    \midrule
    $\cosum$           & \gray{0.196}\\
    $\comean$          & 0.339\\
    $\cogatsum$       & \blue{0.085}\\
    $\comeansum$       & \red{0.079}\\
    \bottomrule
  \end{tabular}
\end{table}

\begin{table*}[h]
\caption{Results on node classification. Top three models are colored by \red{First}, \blue{Second}, \gray{Third}.}
  \centering
  \begin{tabular}{l@{\hspace{2pt}}ccccc}
    \toprule
    & roman-empire & amazon-ratings & minesweeper & tolokers & questions\\
    \midrule
    GCN         & 73.69 \stdfont{$\pm$ 0.74}        & 48.70 \stdfont{$\pm$ 0.63}        & 89.75 \stdfont{$\pm$ 0.52}        & 83.64 \stdfont{$\pm$ 0.67}        & 76.09 \stdfont{$\pm$ 1.27}        \\
    SAGE        & 85.74 \stdfont{$\pm$ 0.67}        & \blue{53.63} \stdfont{$\pm$ 0.39} & 93.51 \stdfont{$\pm$ 0.57}        & 82.43 \stdfont{$\pm$} 0.44        & 76.44 \stdfont{$\pm$ 0.62}        \\
    GAT         & 80.87 \stdfont{$\pm$ 0.30}        & 49.09 \stdfont{$\pm$ 0.63}        & 92.01 \stdfont{$\pm$ 0.68}        & \gray{83.70} \stdfont{$\pm$ 0.47}        & 77.43 \stdfont{$\pm$ 1.20}\\
    GAT-sep     & \gray{88.75} \stdfont{$\pm$ 0.41}        & {52.70} \stdfont{$\pm$ 0.62}        & \gray{93.91} \stdfont{$\pm$ 0.35} & \blue{83.78} \stdfont{$\pm$ 0.43}        & 76.79 \stdfont{$\pm$ 0.71}        \\
    GT          & 86.51 \stdfont{$\pm$ 0.73}        & 51.17 \stdfont{$\pm$ 0.66}        & 91.85 \stdfont{$\pm$ 0.76}        & 83.23 \stdfont{$\pm$ 0.64}        & 77.95 \stdfont{$\pm$ 0.68} \\
    GT-sep      & {87.32} \stdfont{$\pm$ 0.39} & 52.18 \stdfont{$\pm$ 0.80}        & 92.29 \stdfont{$\pm$ 0.47}        & 82.52 \stdfont{$\pm$ 0.92}        & \gray{78.05} \stdfont{$\pm$ 0.93} \\
    \midrule
   $\cosum$   & \red{91.57} \stdfont{$\pm$ 0.32} & 51.28 \stdfont{$\pm$ 0.56} & \blue{95.09} \stdfont{$\pm$ 1.18} & 83.36 \stdfont{$\pm$ 0.89} & \red{80.02} \stdfont{$\pm$ 0.86} \\
    $\comean$  & \blue{91.37} \stdfont{$\pm$ 0.35} & 
     \red{54.17} \stdfont{$\pm$ 0.37} 
    & \red{97.31} \stdfont{$\pm$ 0.41} & \red{84.45} \stdfont{$\pm$ 1.17} & 76.54 \stdfont{$\pm$ 0.95} \\
    \bottomrule
  \end{tabular}
  \label{tab:node-classification}
\end{table*}
 
$\comeansum$: The best result is achieved by this model, because \sumgnn as the action network can accomplish (1) and (2), and \meangnn as the environment network can accomplish (3). Therefore, this model leverages the strengths of $\sumgnn$ and $\meangnn$ to cater to the different roles of the action and environment networks, making it the most natural $\cognn$ model for the regression task.

$\cogatsum$: We observe a very similar phenomenon here to that of $\comeansum$. The action network allows GAT to determine the right topology, and GAT only needs to learn to average the features. This shows the contribution of \cognns is orthogonal to that of attention aggregation.

$\cosum$: This model also performs well, primarily because it uses \sumgnn as the action network, accomplishing (1) and (2). However, it uses another \sumgnn as the environment network which cannot easily mimic the averaging of the neighbor's features.

$\comean$: This model clearly performs weakly, since \meangnn as an action network cannot achieve (1) hindering the performance of the whole task. Indeed, $\comean$ performs comparably to $\meangnn$ suggesting that the action network is not useful in this case.

To shed light on the performance of $\cognn$ models, we computed the percentage of edges which are accurately retained or removed by the action network in a single layer $\cognn$ model. We observe an accuracy of 99.71\% for $\comeansum$, 99.55\% for $\cosum$, and 57.20\% for $\comean$. This empirically confirms the expected behavior of \cognns. In fact, the example tree is shown on the right of \Cref{fig:combined_synthetic} is taken from the experiment with $\comeansum$: reassuringly, this model learns precisely the actions that induce the shown optimal subgraph.

\subsection{Node Classification with Heterophilic Graphs}
\label{subsec:node-classification}

One of the strengths of $\cognns$ is their capability to utilize task-specific information propagation, which raises an obvious question: \emph{could \cognns outperform the baselines on heterophilious graphs, where standard message passing is known to suffer?} To answer this question, we assess the performance of $\cognns$ on heterophilic node classification datasets from \citep{platonov2023critical}. 

\textbf{Setup.} We evaluate $\sumgnn$, $\meangnn$ and their $\cognn$ counterparts, $\cosum$ and $\comean$ on the 5 heterophilic graphs, following the 10 data splits and the methodology of \citet{platonov2023critical}. We report the accuracy and standard deviation for roman-empire and amazon-ratings. We also report the ROC AUC and standard deviation for minesweeper, tolokers, and questions. The classical baselines GCN, GraphSAGE, GAT, GAT-sep, GT \citep{shi2021masked} and GT-sep are taken from \citet{platonov2023critical}. We use the Adam optimizer and report all hyperparameters in \Cref{app:hyperparameters}.

\textbf{Results.} All results are reported in \Cref{tab:node-classification}. Observe that \cognns achieve state-of-the-art results across the board, despite using relatively simple architectures as their action and environment networks. Importantly, $\cognns$ demonstrate an average accuracy improvement of 2.23\% compared to all baseline methods, across all datasets, surpassing the performance of more complex models such as GT.
In our main finding  we observe a consistent trend: enhancing standard models with action networks of \cognns results in improvements in performance. For example, we report $3.19\%$ improvement in accuracy on the roman-empire and $3.62\%$ improvement in ROC AUC on minesweeper compared to the best performing baseline. This shows that $\cognns$ are flexible and effective on different datasets and tasks. These results are reassuring as they establish $\cognns$ as a strong method in the heterophilic setting due to its unique ability to manipulate information flow.

\section{Empirical Insights for the Actions}
The action network of $\cognns$ is the key model component. The purpose of this section is to provide additional insights regarding the actions being learned by \cognns.

\subsection{Actions on Heterophilic vs Homophilic Graphs}
\label{sec:hh}
We aim to compare the actions learned on a homophilic task to the actions learned on a heterophilic task. One idea would be to inspect the learned action distributions, but they alone may not provide a clear picture of the graph's topology. For example, two connected nodes that choose to $\isolatelong$ achieve the same topology as nodes that choose both to $\broadcastlong$ or $\listenlong$. This is a result of the immense number of action configurations and their interactions.

To better understand the learned graph topology, we inspect the \emph{induced directed graphs} at every layer. Specifically, we present the \emph{ratio of the directed edges that are kept} across the different layers in \cref{fig:edge_ratio}. We record the directed edge ratio over the 10 different layers of our best, fully trained 10 $\comean$ models on the roman-empire \citep{platonov2023critical} and cora datasets \citep{pei2020geomgcn}. We follow the 10 data splits and the methodology of \citet{platonov2023critical} and \citet{Yang16}, respectively.

\begin{figure}[t!]
    \centering
        \centering
        \includegraphics[width=0.8\linewidth]{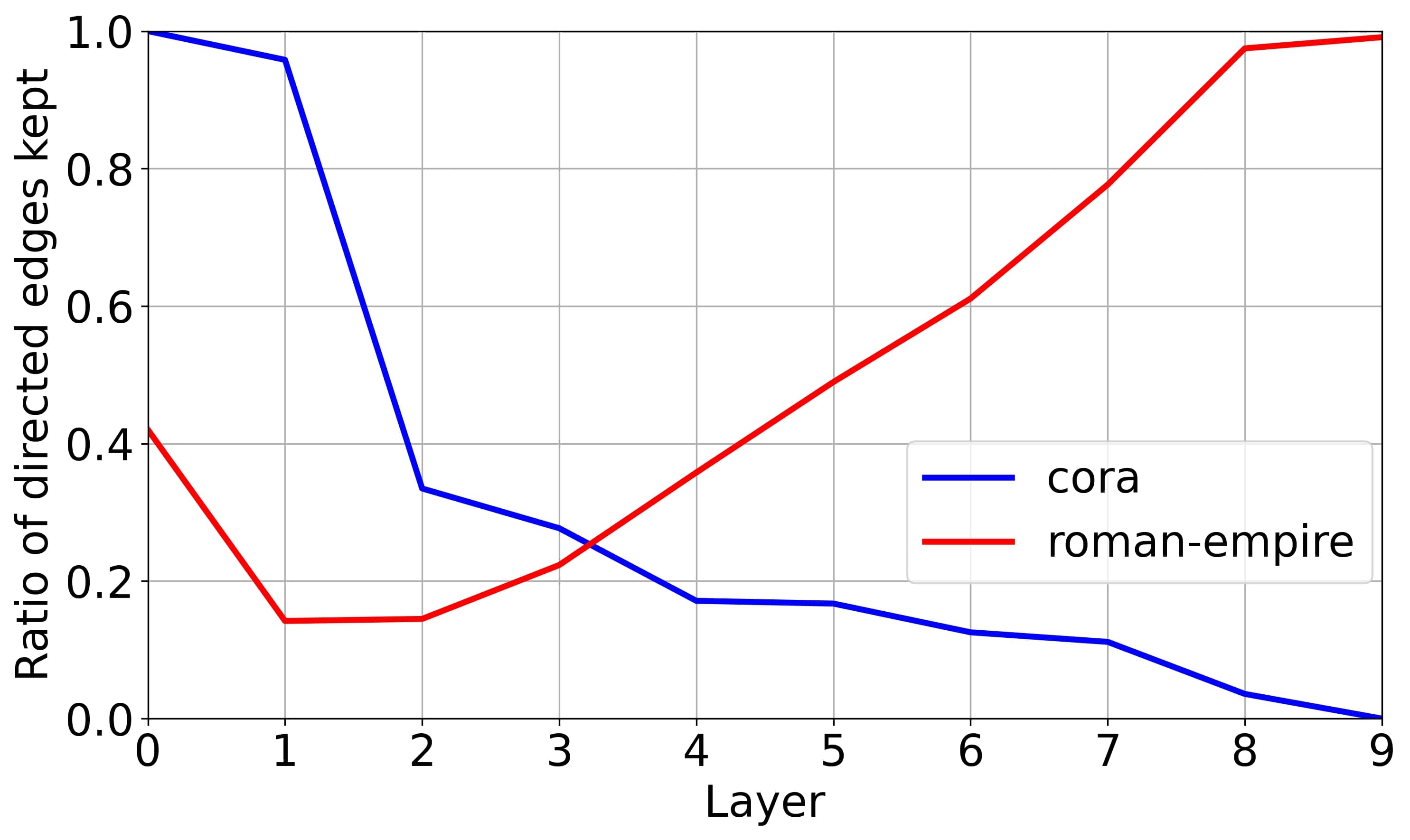}
    \caption{The ratio of directed edges that are kept on cora (as a homophilic dataset) and on roman-empire (as a heterophilic dataset) for each layer  $0 \leq \ell <10$.}
    \label{fig:edge_ratio}
\end{figure}

This experiment serves as a strong evidence for the adaptive nature of \cognns this statement. Indeed, by inspecting \Cref{fig:edge_ratio}, we observe completely \emph{opposite trends} between the two datasets.

On the homophilic dataset cora, the ratio of edges that are kept gradually \emph{decreases} as we go to the deeper layers. In fact, $100\%$ of the edges are kept at layer $\ell=0$ while \emph{all} edges are dropped at layer $\ell=9$. This is very insightful because homophilic datasets are known to \emph{not} benefit from using many layers, and the trained \cognn model recognizes this by eventually isolating all the nodes. This is particularly the case for cora, where classical MPNNs typically achieve their best performance with $1$-$3$ layers.

On the heterophilic dataset roman-empire, the ratio of edges that are kept gradually \emph{increases} after $\ell=1$ as we go to the deeper layers. Initially, $\sim 42 \%$ of the edges are kept at layer $\ell=0$ while eventually this reaches $99\%$ at layer $\ell=9$. This is interesting, since in heterophilic graphs, edges tend to connect nodes of different classes and so classical MPNNs, which aggregate information based on the homophily assumption perform poorly. Although, $\cognn$ model uses these models it compensates by controlling information flow. The model manages to capture the heterophilous aspect of the dataset by restricting the flow of information in the early layer and slowly enabling it the deeper the layer (the further away the nodes), which might be a great benefit to its success over heterophilic benchmarks. 

\subsection{What Actions are Performed on Minesweeper?}
\label{subsec:visualize}
To better understand the topology learned by $\cognns$, we visualize the topology at each layer in a $\cognn$ model over the highly regular minesweepers dataset. 

\textbf{Dataset.}
Minesweeper \citep{platonov2023critical} is a synthetic
dataset inspired by the Minesweeper game. It is a semi-supervised node classification dataset with a regular $100 \times 100$ grid where each node is connected to eight neighboring nodes. Each node has an one-hot-encoded input feature showing the number of adjacent mines. A randomly chosen 50\% of the nodes have an unknown feature, indicated by a separate binary feature. The task is to identify whether the querying node is a mine.

\textbf{Setup.}
We train a $10$-layer $\comean$ model and present the graph topology at layer $\ell=4$. The evolution of the graph topology from layer $\ell=1$ to layer $\ell=8$ is presented in \Cref{app:visu_actions}. We choose a node (black), and at every layer $\ell$, we depict its neighbors up to distance $10$.  In this visualization, nodes which are mines are shown in red, and other nodes in blue. The features of non-mine nodes (indicating the number of neighboring mines) are shown explicitly whereas the nodes whose features are hidden are labeled with a question mark.

\begin{figure}[t]
    \centering
        \centering
        \includegraphics[width=0.85\linewidth]{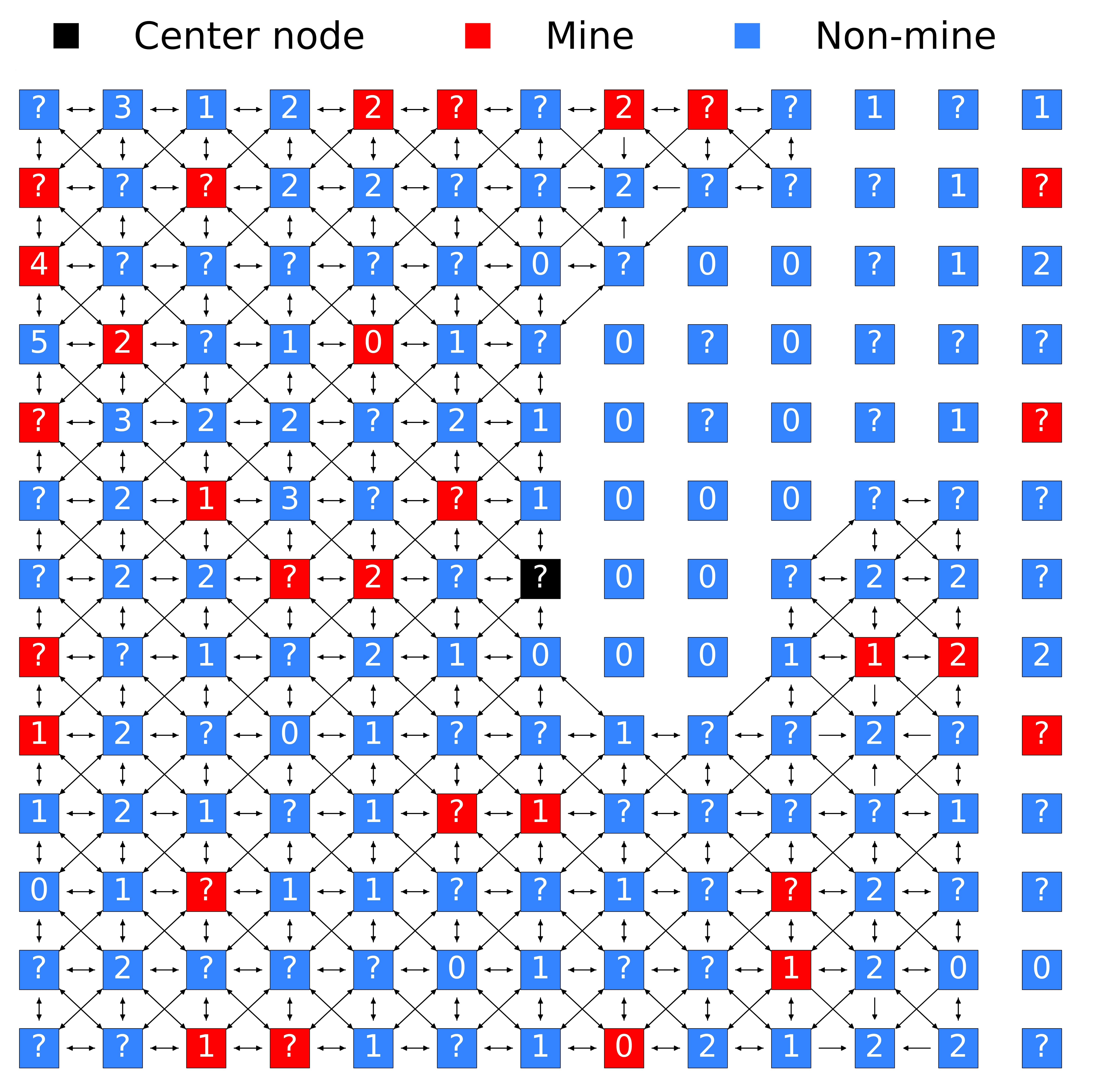}
    \caption{The 10-hop neighborhood at layer $\ell=4$}
    \label{fig:minesweeper}
\end{figure}

\begin{figure*}[ht!]
    \begin{subfigure}{0.5\textwidth}
        \centering
        \includegraphics[width=0.7\linewidth]{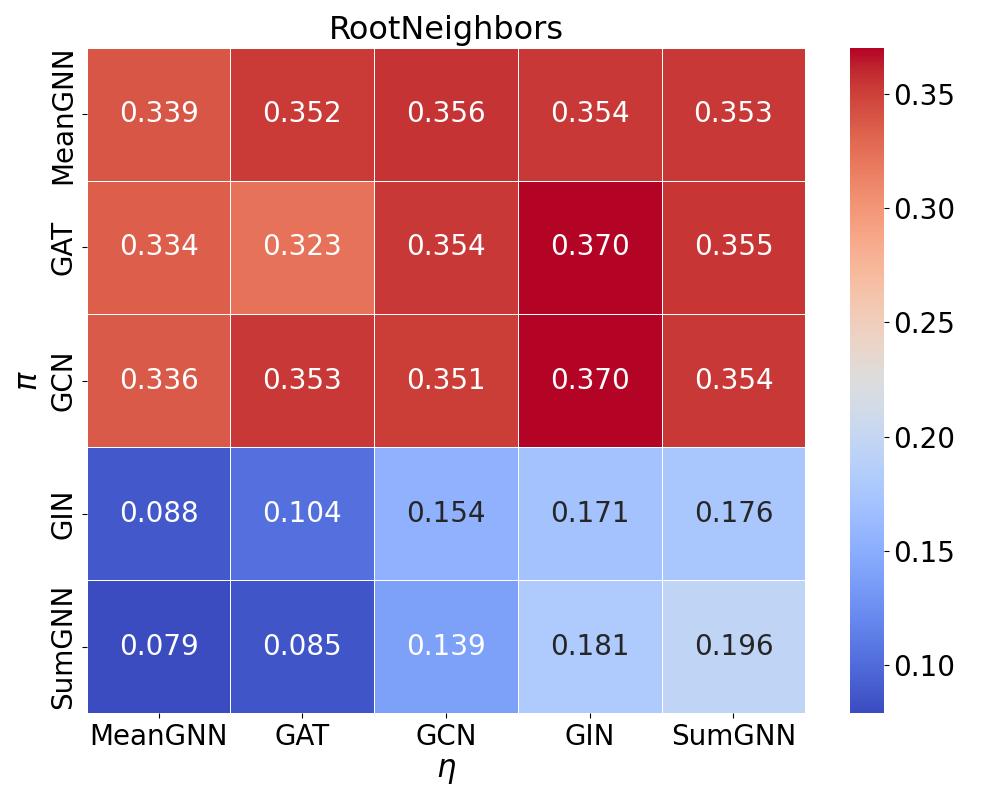}
    \end{subfigure}
        \hfill
    \begin{subfigure}{0.5\textwidth}
        \centering
        \includegraphics[width=0.7\linewidth]{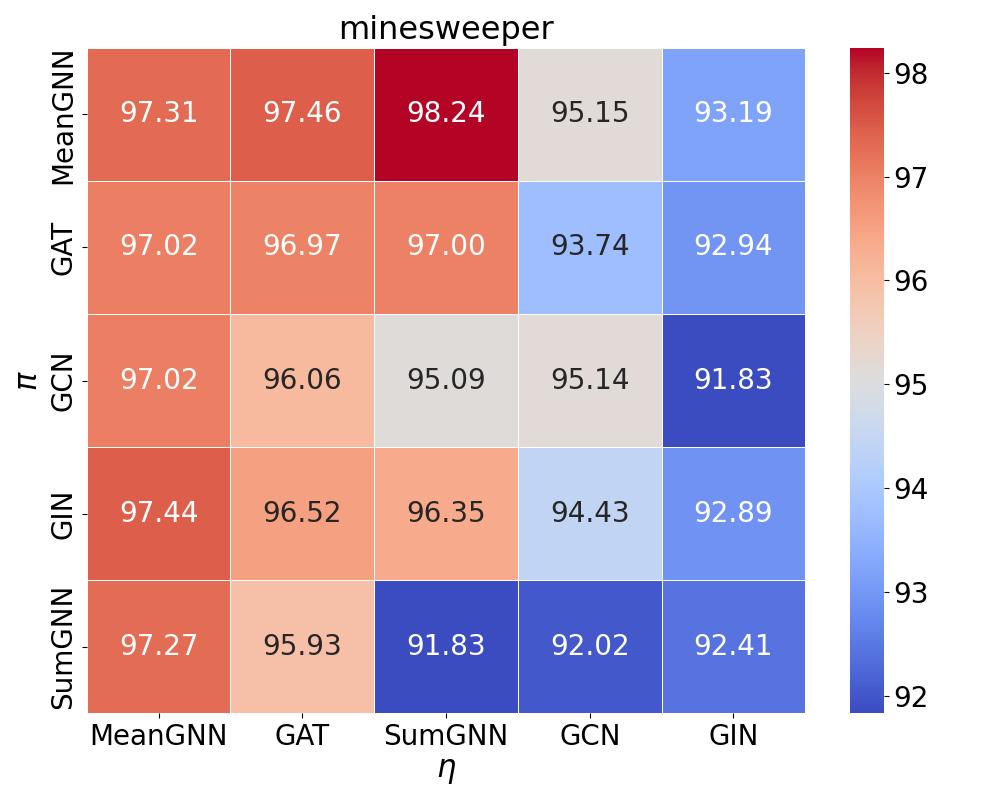}
    \end{subfigure}
    \caption{MAE and ROC AUC as a function of the choice of action ($\action$) and environment ($\env$) networks over \rootn (\textbf{left}) and the minesweeper (\textbf{right}) experiments, respectively.}
    \label{fig:heatmaps}
\end{figure*}

\textbf{Interpreting the Actions.} The visualization of the actions at layer $\ell=4$ is shown in \Cref{fig:minesweeper}. In this game, every label is informative: non-$0$-labeled nodes are more informative in the earlier layers ($\ell=1,2,3,4$), whereas $0$-labeled nodes become more informative at later layers. This can be explained by considering two cases:
\begin{enumerate}[noitemsep,topsep=0pt,parsep=1pt,partopsep=0pt,leftmargin=*]
    \item The target node has \emph{at least one} $0$-labeled neighbor: In this case, the prediction trivializes, since we know that the target node cannot be a mine. \label{item1}
    \item The target node has \emph{no} $0$-labeled neighbors: In this case, the model needs to make a sophisticated inference based on the surrounding mines within the $k$-hop distance. In this scenario, a node obtains more information by aggregating from non-$0$-label nodes. The model can still implicitly infer ``no mine'' from the lack of a signal/message. \label{item2}
\end{enumerate} 
The action network appears to diffuse the information of non-$0$-labeled nodes in the early layers to capture nodes from case (\ref{item2}), while largely isolating the $0$-labeled nodes until the last layers (avoiding mixing of signals and potential loss of information) which can later be used to capture the nodes from case (\ref{item1}).
In the earlier layers, the action network prioritizes the information flowing from the left sections of the grid in \Cref{fig:minesweeper} where more mines are present. After identifying the most crucial information and propagating this through the network, it then requires this information to also be communicated with the nodes that initially were labeled with $0$. This leads to an almost fully connected grid in the later layers (see $\ell=7,8$ in \Cref{fig:visualization7,fig:visualization8}).

\subsection{Which Action or Environment Network?}
\label{subsec:ablation}
We conduct an ablation study on the choice of action and environment networks to quantify its affect.

\textbf{Setup.} We experiment with all 25 combinations of the architectures \meangnns, GAT, \sumgnn, GCN and GIN on the heterophilic graph minesweeper and on the synthetic dataset \rootn. We report MAE for \rootn and the ROC AUC for minesweeper.

\textbf{\rootn Results.} The results reported in \cref{fig:heatmaps} (left) support our analysis from \cref{subsec:synthetic}: an environment network with mean-type aggregation (GCN, GAT, or \meangnn) and an action network with sum-type aggregation (\sumgnn, or GIN) are best choices for the task. The choice of the action network is critical for this task: \sumgnn and GIN yield best results across the board when they are used as action networks. In contrast, if we use GAT, GCN, or \meangnn as action networks then the results are poor and  comparable to baseline results on this task. These action networks cannot detect node cardinality which prevents them from choosing the optimal actions as elaborated in \cref{subsec:synthetic}. The choice of the environment network is relatively less important for this task.

\textbf{Minesweeper Results.} In \cref{fig:heatmaps} (right), \cognns achieve multiple state-of-the-art results on minesweeper with different choices of action and environment networks.

In terms of the environment network, we observe that MeanGNN and GAT yield consistently robust results when used as environment networks regardless of the choice of the action network. This makes sense in the context of the minesweeper game. In order to make a good inference, a $k$-layer environment network can keep track of the average number of mines found in each hop distance. The task is hence well-suited to mean-style aggregation environment networks, which manifests as empirical robustness. GCN performs worse as environment network, since it cannot distinguish a node from its neighbors.

In terms of the action network, we observed earlier that the role of the action network is to mainly make a distinction between $0$-labeled nodes and non-$0$-labeled nodes, and such an action network can be realized with all of the architectures considered.
As a result, we do not observe dramatic differences in the performance regarding to the choice of the action network in this task.

\section{Summary and Outlook}

We introduced $\cognn$ architectures which can dynamically explore the graph topology while learning. These architectures have desirable properties which can inform future work. Looking forward, one potential future direction is to adapt \cognns to other types of graphs such as directed, and even multi-relational graphs. One possible approach is by including actions that also consider the directionality. For example, for each node $u$, one can define the actions \textsc{Listen-Inc} (listen to nodes that have an incoming edge to $u$) and \textsc{Listen-Out} (listen to nodes that have an incoming edge from $u$) and extend the other actions analogously to incorporate directionality. It is also possible to consider other types of actions or extend our approach to edge-wise actions (rather than node-wise), though these extensions will lead to a larger state space. 

\section*{Acknowledgments}
The authors would like to thank the anonymous reviewers for their feedback which led to substantial improvements in the presentation of the paper. The first author is funded by the Clarendon scholarship. The authors would like to also acknowledge the use of the University of Oxford Advanced Research Computing (ARC) facility in carrying out this work (http://dx.doi.org/10.5281/zenodo.22558). This work was also partially funded by EPSRC Turing AI World-Leading Research Fellowship
No. EP/X040062/1.

\section*{Impact Statement}
This paper presents a novel graph neural network paradigm whose goal is to explore the graph topology while learning. There are many potential societal consequences of our work, none of which we feel must be specifically highlighted here.

\bibliography{icml2024}
\bibliographystyle{icml2024}

\newpage
\appendix
\onecolumn
\section{Proofs of Technical Results}
\label{app:proofs}

\subsection{Proof of \cref{prop:expresivity}}
In order to prove \Cref{prop:expresivity}, we first prove the following lemma, which shows that all non-isolated nodes of an input graph can be individualized by \cognns:

\begin{lemma}\label{lemma:node-expresivity}
    Let $G=(V,E,\mX)$ be a graph with node features. For every pair of non-isolated nodes $u,v\in V$ and for all $\delta > 0$, there exists a $\cognn$ architecture with sufficiently many layers $L$ which satisfies $\mathbb{P}(\vh_u^{\left(L\right)} \neq \vh_v^{\left(L\right)}) \ge 1 - \delta$.
\end{lemma}
\begin{proof}
    We consider an $L$-layer $\cognn(\env,\action)$ architecture satisfying the following:
    \begin{enumerate}[(i)]
        \item the \emph{environment} network $\env$ is composed of ${L}$ injective layers,
        \item the \emph{action} network $\action$ is composed of a \emph{single} layer, and it is shared across \cognn layers.
    \end{enumerate}
    Item (i) can be satisfied by a large class of GNN architectures, including \sumgnn \citep{MorrisAAAI19} and GIN \citep{xu18}. We start by assuming $\vh_u^{(0)} = \vh_v^{(0)}$. These representations can be differentiated if the model can jointly realize the following actions at some layer $\ell$ using the action network $\action$:
    \begin{enumerate}
        \item $a_u^{\left(\ell \right)} =\listen \lor \standard$,
        \item $a_v^{\left(\ell \right)} =\isolate \lor \broadcast$, and
        \item $\exists$ a neighbor $w$ of $u$ s.t.\ $a_w^{\left(\ell \right)} =\standard \lor \broadcast$.
    \end{enumerate}
    The key point is to ensure an update for the state of $u$ via an aggregated message from its neighbors (at least one), while isolating $v$. In what follows, we assume the worst-case for the degree of $u$ and consider a node $w$ to be the only neighbor of $u$. Let us denote the joint probability of realizing these actions 1-3 for the nodes $u,v,w$ at layer $\ell$ as:
    \[
    p^\stepl_{u,v} = \mathbb{P}\left( 
            ( a_u^{\left(\ell \right)} =\listen \lor \standard ) \land 
            (a_v^{\left(\ell \right)} =\isolate \lor \broadcast)  \land
            (a_w^{\left(\ell \right)} =\standard \lor \broadcast) \right).
    \]
    The probability of taking each action is non-zero (since it is a result of applying softmax) and $u$ has at least one neighbor (non-isolated), therefore $p^\stepl_{u,v} >0$. For example, if we assume a constant action network that outputs a uniform distribution over the possible actions (each action probability $0.25$) then $p^\stepl_{u,v} = 0.125$. 

    This means that the environment network $\env$ applies the following updates to the states of $u$ and $v$ with probability $p^\stepl_{u,v} >0$:
    \begin{align*}
          \vh_u^\steplplus &= 
            \env^\stepl\left(
                \vh_u^\stepl,
                 \{\!\!\{  \vh_w^\stepl \mid w \in \mathcal{N}_u,a_w^{\left(\ell \right)} =\standard \lor \broadcast \}\!\!\} \right), \\
        \vh_v^\steplplus &= 
            \env^\stepl\left(
                \vh_v^\stepl,\{\!\!\{ \}\!\!\} 
                \right).
    \end{align*} 
    The inputs to the environment network layer $\env^{(\ell)}$ for these updates are clearly different, and since the environment layer is injective, we conclude that $\vh_u^\steplplus \neq \vh_v^\steplplus$.
    
    Thus, the probability of having different final representations for the nodes $u$ and $v$ is lower bounded by the probability of the events 1-3 jointly occurring at least once in one of the \cognn layers, which, by applying the union bound, yields:
    \begin{align*}
        \mathbb{P}(\vh_u^{\left(L_{u,v}\right)} \neq \vh_v^{\left(L_{u,v}\right)}) \ge 1 - \prod^{L_{u,v}}_{\ell=0} \left(1 - p^\stepl_{u,v}\right) \ge 1 - \left(1 - \gamma_{u,v} \right)^{L_{u,v}} \ge 1 - \delta
    \end{align*}
    where $\gamma_{u,v} = \max_{\ell\in [L_{u,v}]}{\left(p^\stepl_{u,v}\right)}$ and $L_{u,v}=\log_{1 - \gamma_{u,v}}\left(\delta\right)$.
 
    We repeat this process for all pairs of non-isolated nodes $u,v\in V$. Due to the injectivity of $\env^{(\ell)}$ for all $\ell\in [L]$, once the nodes are distinguished, they cannot remain so in deeper layers of the architecture, which ensures that all nodes $u,v\in V$ differ in their final representations $\vh_u^\stepl\neq\vh_v^\stepl$ after this process completes. The number of layers required for this construction is then given by:
    \begin{equation*}
        L=  |V\setminus I| \log_{1 - \alpha}\left(\delta\right) \ge \sum_{u,v\in V\setminus I} \log_{1 - \gamma_{u,v}}\left(\delta\right),
    \end{equation*}
    where $I$ is the set of all isolated nodes in $V$ and
    \begin{equation*}
        \alpha= \max_{u,v\in V\setminus I}\left(\gamma_{u,v}\right)= \max_{u,v\in V\setminus I}\left(\max_{\ell\in [L_{u,v}]}{\left(p^\stepl_{u,v}\right)}\right).
    \end{equation*}
Having shown a \cognn construction with the number of layers bounded as above, we conclude the proof.   
\end{proof}

\begin{propositioncopy}{\ref{prop:expresivity}}
Let  $G_1=(V_1,E_1,\mX_1)$ and $G_2=(V_2,E_2,\mX_2)$ be  two non-isomorphic graphs. Then, for any threshold $0 < \delta < 1$, there exists a parametrization of a $\cognn$ architecture using sufficiently many layers $L$, satisfying $\mathbb{P}(\vz_{G_1}^{\left(L\right)} \neq \vz_{G_2}^{\left(L\right)}) \ge 1 - \delta$.
\end{propositioncopy}
\begin{proof}
    Let $\delta>0$ be any value and consider the graph $G=(V,E,\mX)$ which has $G_1$ and $G_2$ as its components:
    \begin{equation*}
        V= V_1\cup V_2, \quad E= E_1\cup E_2, \quad \mX = \mX_1 || \mX_2,
    \end{equation*}
    where $||$ is the matrix horizontal concatenation. By \cref{lemma:node-expresivity}, for every pair of non-isolated nodes $u,v\in V$ and for all $\delta > 0$, there exists a $\cognn$ architecture with sufficiently many layers $L=|V\setminus I| \log_{1 - \alpha}\left(\delta\right)$ which satisfies:
    \begin{equation*}
      \mathbb{P}(\vh_u^{\left(L_{u,v}\right)} \neq \vh_v^{\left(L_{u,v}\right)}) \ge 1 - \delta, \text{ with }         \alpha= \max_{u,v\in V\setminus I}\left(\max_{\ell\in [L_{u,v}]}{\left(p^\stepl_{u,v}\right)}\right),
    \end{equation*}
    where $p^\stepl_{u,v}$ represents a lower bound on the probability for the representations of nodes $u,v\in V$ at layer $\ell$ being different.

    We use the same \cognn construction given in  \cref{lemma:node-expresivity} on $G$, which ensures that all non-isolated nodes have different representations in $G$. When applying this \cognn to $G_1$ and $G_2$ separately, we get that every non-isolated node from either graph has a different representation with probability $1-\delta$ as a result. Hence, the multiset $\mathcal{M}_1$ of node features for $G_1$ and the multiset $\mathcal{M}_2$ of node features of $G_2$ must differ. Assuming an injective pooling function from these multisets to graph-level representations, we get:
    \begin{equation*}
        \mathbb{P}(\vz_{G_1}^{\left(L\right)} \neq \vz_{G_2}^{\left(L\right)}) \ge 1 - \delta
    \end{equation*}
    for $L=|V\setminus I| \log_{1 - \alpha}\left(\delta\right)$.
\end{proof}

\subsection{Proof of \cref{prop:long-range}}

\begin{propositioncopy}{\ref{prop:long-range}}
Let $G=(V,E,\mX)$ be a connected graph with node features. For some $k>0$, for any target node $v \in V$, for any $k$ source nodes $u_1, \ldots, u_k \in V$, and for any compact, differentiable function $f: \mathbb{R}^{d^{(0)}} \times \ldots \times \mathbb{R}^{d^{(0)}} \to \mathbb{R}^d$, there exists an $L$-layer $\cognn$ computing final node representations such that for any $\epsilon, \delta >0$ it holds that $\mathbb{P} (|\vh_v^{(L)} - f(\vx_{u_1}, \ldots \vx_{u_k})| < \epsilon) \geq 1-\delta$.
\end{propositioncopy}
\begin{proof}
    For arbitrary $\epsilon,\delta>0$, we start by constructing a feature encoder $\operatorname{ENC}: \sR^{d^{(0)}}\rightarrow\sR^{2(k + 1){d^{(0)}}}$ which encodes the initial representations $\vx_w\in\sR^{d^{(0)}}$ of each node $w$ as follows:
    \[
    \operatorname{ENC}(\vx_w) = \big[ \underbrace{\Tilde{\vx}_w^\top \oplus \ldots \oplus \Tilde{\vx}_w^\top}_{k+1} \big]^\top,
    \]
    where $\Tilde{\vx}_w= [\operatorname{ReLU}(\vx_w^\top) \oplus \operatorname{ReLU}(-\vx_w^\top)]^\top$.
    Observe that this encoder can be parametrized using a 2-layer MLP, and that $\Tilde{\vx}_w$ can be decoded using a single linear layer to get back to the initial features:
    \[
    \operatorname{DEC}(\Tilde{\vx}_w) = \operatorname{DEC} \left(\big[\operatorname{ReLU}(\vx_w^\top) \oplus \operatorname{ReLU}(-\vx_w^\top)\big]^\top \right) = \vx_w  
    \]

    \textbf{Individualizing the Graph.} Importantly, we encode the features using ${2(k + 1){d^{(0)}}}$ dimensions in order to be able to preserve the original node features. Using the construction from \cref{lemma:node-expresivity}, we can ensure that every pair of nodes in the connected graph have different features with probability $1-\delta_1$. However, if we do this na\"ively, then the node features will be changed before we can transmit them to the target node. We therefore make sure that the width of the \cognn architecture from  \cref{lemma:node-expresivity} is increased to ${2(k + 1){d^{(0)}}}$ dimensions such that it applies the identity mapping on all features beyond the first $2d^{(0)}$ components. This way we make sure that all feature components beyond the first ${2{d^{(0)}}}$ components are preserved. The existence of such a \cognn is straightforward since we can always do an identity mapping using base environment models such \sumgnns. We use $L_1$ \cognn layers for this part of the construction.
    
    In order for our architecture to retain a positive representation for all nodes, we now construct 2 additional layers which encode the representation $\vh_w^{(L)}\in\sR^{2(k + 1)d^{(0)}}$ of each node $w$ as follows:
    \[
     [\operatorname{ReLU}(\vq_w^\top) \oplus \operatorname{ReLU}(-\vq_w^\top) \oplus \Tilde{\vx}_w^\top \oplus \ldots \oplus \Tilde{\vx}_w^\top]^\top
    \]
    where $\vq_w\in\sR^{2d^{(0)}}$ denotes a vector of the first $2d^{(0)}$ entries of $\vh_w^{(L_1)}$.
    
    \textbf{Transmitting Information.} Consider a shortest path $u_1 = w_0 \rightarrow w_1 \rightarrow \cdots \rightarrow w_r \rightarrow w_{r+1} = v$ of length $r_1$ from node $u_1$ to node $v$. We use exactly $r_1$ \cognn layers in this part of the construction. For the first these layers, the action network assigns the following actions to these nodes: 
    \begin{enumerate}[$\bullet$]
        \item $w_0$ performs the action $\broadcastlong$, 
        \item $w_{1}$ performs the action $\listenlong$, and
        \item  all other nodes are perform the action $\isolatelong$. 
    \end{enumerate}
    This is then repeated in the remaining layers, for all consecutive pairs $w_i$, $w_{i+1}$, $0 \leq i \leq r$ until the whole path is traversed. That is, at every layer, all graph edges are removed except the one between $w_i$ and $w_{i+1}$, for each $0 \leq i \leq r$. By construction each element in the node representations is positive and so we can ignore the ReLU.
    
    We apply the former construction such that it acts on entries $2d^{(0)}$ to $3d^{(0)}$ of the node representations, resulting in the following representation for node $v$:
    \[
    [\operatorname{ReLU}(\vq_w^\top) \oplus \operatorname{ReLU}(-\vq_w^\top) \oplus \Tilde{\vx}_{u_1} \oplus \Tilde{\vx}_{v} \oplus \ldots \oplus \Tilde{\vx}_{v}]
    \]
    where $\vq_w\in\sR^{2d^{(0)}}$ denotes a vector of the first $2d^{(0)}$ entries of $\vh_w^{(L_1)}$.
    
    We denote the probability in which node $y$ does not follow the construction at stage $1 \leq t \leq r$ by $\beta_y^{(t)}$ such that the probability that all graph edges are removed except the one between $w_i$ and $w_{i+1}$ at stage $t$ is lower bounded by $\left(1 - \beta\right)^{\lvert V\rvert}$, where $\beta=\max_{y\in V}(\beta_y)$. Thus, the probability that the construction holds is bounded by $\left(1 - \beta\right)^{|V| r_1}$.

    The same process is then repeated for nodes $u_i$, $2\leq i\leq k$, acting on the entries $(k+1)d^{(0)}$ to $(k + 2)d^{(0)}$ of the node representations and resulting in the following representation for node $v$:
    \[
    [\operatorname{ReLU}(\vq_w^\top) \oplus \operatorname{ReLU}(-\vq_w^\top) \oplus \Tilde{\vx}_{u_1} \oplus \Tilde{\vx}_{u_2} \oplus \ldots \oplus \Tilde{\vx}_{u_k}]
    \]
    
    In order to decode the positive features, we construct the feature decoder $\operatorname{DEC}': \sR^{2(k + 2){d^{(0)}}}\rightarrow\sR^{(k + 1){d^{(0)}}}$, that for $1\leq i\leq k$ applies $\operatorname{DEC}$ to entries $2(i+1)d^{(0)}$ to $(i + 2)d^{(0)}$ of its input as follows:
    \[
    [\operatorname{DEC}(\Tilde{\vx}_{u_1})\oplus\ldots\oplus\operatorname{DEC}(\Tilde{\vx}_{u_k})] = [\vx_{u_1}\oplus \ldots\oplus \vx_{u_k}]
    \]
    
    Given $\epsilon, \delta$, we set:
    \[\delta_2 = 1- \frac{1 - \delta}{(1 - \delta_1)\left(1 - \beta\right)^{|V| \sum_{i=1}^{k+1} r_i}} >0.\]
    
    Having transmitted and decoded all the required features into $\vx = [\vx_1\oplus \ldots \oplus \vx_k]$, where $\vx_i$ denotes the vector of entries $id^{(0)}$ to $(i + 1)d^{(0)}$ for $0\leq i \leq k$, we can now use an $\operatorname{MLP}: \sR^{(k + 1){d^{(0)}}}\rightarrow\sR^{d}$ and the universal approximation property to map this vector to the final representation $\vh_v^{(L)}$ such that:  
    \begin{align*}
        \mathbb{P} (|\vh_v^{(L)} - f(\vx_{u_1}, \ldots \vx_{u_k})| < \epsilon) &\ge (1- \delta_1)\left(1 - \beta\right)^{|V| \sum_{i=1}^{k+1} r_i}(1- \delta_2)\ge 1 - \delta.
    \end{align*}
The construction hence requires  $\left(L=L_1 + 2 + \sum_{i=0}^k r_i\right)$ $\cognn$ layers.
\end{proof}

\section{Relation to Over-squashing}
\label{app:oversquashing}

Over-squashing refers to the failure of message passing to propagate information on the graph. \citet{topping2022understanding} and \citet{di2023over} formalized over-squashing as the insensitivity of an $r$-layer MPNN output at node $u$ to the input features of a distant node $v$, expressed through a bound on the Jacobian
$
\|\partial \boldsymbol{h}_v^{(r)}/\partial \boldsymbol{x}_u\| \leq C^{r}(\hat{\boldsymbol{A}}^{r})_{v u}
$, 
where $C$ encapsulated architecture-related constants (e.g., width, smoothness of the activation function, etc.) and the normalized adjacency matrix $\hat{\boldsymbol{A}}$ captures the effect of the graph.  
Graph rewiring techniques amount to modifying $\hat{\boldsymbol{A}}$ so as to increase the upper bound and thereby reduce the effect of over-squashing.

Observe that the actions of every node in \cognns result in an effective graph rewiring (different at every layer). As a result, the action network can choose actions that transmit the features of node $u\in V$ to node $v\in V$ as shown in \cref{prop:long-range}, resulting in the maximization of the bound on the Jacobian between a pair of nodes or ($k$ nodes, for some fixed $k$).

\section{Additional Experiments}
\label{app:additional-experiments}
\subsection{Expressivity Experiment}
\label{app:cycles}
In \Cref{prop:expresivity} we state that $\cognns$ can distinguish between pairs of graphs which are 1-WL indistinguishable. We validate this with a simple synthetic dataset: \cycles. \cycles consists of $7$ pairs of undirected graphs, where the first graph is a $k$-cycle for $k \in [6,12]$ and the second graph is a disjoint union of a $(k{-}3)$-cycle and a triangle. The train/validation/test set are the $k\in[6,7]/[8,9]/[10,12]$ pairs, correspondingly. The task is to correctly identify the cycle graphs. As the pairs are 1-WL indistinguishable, solving this task implies a strictly higher expressive power than 1-WL. 

Our main finding is that $\cosum$ and $\comean$ achieve $100\%$ accuracy, perfectly classifying the cycles, whereas their corresponding classical $\sumgnn$ and $\meangnn$ achieve a random guess accuracy of $50 \%$. These results imply that \cognn can increase the expressive power of their classical counterparts. We find the model behaviour rather volatile during training, which necessitated  careful tuning of hyperparameters.

\subsection{Long-range Interactions}
\label{app:long-range-exp}

\begin{wraptable}{r}{0.35\textwidth} 
\caption{Results on LRGB. Top three models are colored by \red{First}, \blue{Second}, \gray{Third}.}
\label{tab:lrgb_result}
   \centering
  \begin{tabular}{lcc}
    \toprule
    & Peptides-func \\
    \midrule
    GCN & 0.6860 \stdfont{$\pm$ 0.0050}\\
    GINE & 0.6621 \stdfont{$\pm$ 0.0067}\\
    GatedGCN & 0.6765 \stdfont{$\pm$ 0.0047}\\
    CRaWl & \blue{0.7074} \stdfont{$\pm$ 0.0032}\\
    DRew & \red{0.7150} \stdfont{$\pm$ 0.0044}\\
    Exphormer & 0.6527 \stdfont{$\pm$ 0.0043}\\
    GRIT & 0.6988 \stdfont{$\pm$ 0.0082}\\
    Graph-ViT & 0.6942 \stdfont{$\pm$ 0.0075}\\
    G-MLPMixer & 0.6921 \stdfont{$\pm$ 0.0054}\\
    \midrule
    $\cogcn$  & \gray{0.6990} \stdfont{$\pm$ 0.0093}\\
    $\cogin$  & 0.6963 \stdfont{$\pm$ 0.0076}\\
    \bottomrule
    \end{tabular}
\end{wraptable}
To validate the performance of \cognns on long-range tasks, we experiment with the LRGB benchmark~ \citep{dwivedi2023long}.

\textbf{Setup.} We train $\cogcn$ and $\cogin$ $\cogin$ on LRGB and report the unweighted mean Average Precision (AP) for Peptides-func. All experiments are run 4 times with 4 different seeds and follow the data splits provided by \citet{dwivedi2023long}. Following the methodology of \citet{tonshoff23}, we used AdamW as optimizer and cosine-with-warmup scheduler. We also use the provided results for GCN, GCNII \citep{chen2020simple}, GINE, GatedGCN \citep{bresson2018residual}, CRaWl \citep{onshoff2023walking}, DRew \citep{gutteridge2023drew}, Exphormer \citep{Shirzad2023ExphormerST}, GRIT \citep{GRIT}, Graph-ViT / G-MLPMixer \citep{he2023graphvitmlp}.

\textbf{Results.} We follow \citet{tonshoff23} who identified that the previously reported large performance gaps between classical MPNNs and transformer-based models can be closed by a more extensive tuning of MPNNs. In light of this, we note that the performance gap between different models is not large. Classical MPNNs such as GCN, GINE, and GatedGCN surpass some transformer-based approaches such as Exphormer. $\cogcn$ further improves on the competitive GCN and is the third best performing model after DRew and CRaWl. Similarly, $\cogin$ closely matches $\cogcn$ and is substantially better than its base architecture GIN. This experiment further suggests that exploring different classes of \cognns is a promising direction, as \cognns typically boost the performance of their underlying base architecture.

\subsection{Graph Classification}
\label{app:graph-classification}
In this experiment, we evaluate $\cognns$ on the TUDataset \citep{Morris2020} graph classification benchmark.

\textbf{Setup.} We evaluate $\cosum$ and $\comean$ on the 7 graph classification benchmarks, following the risk assessment protocol of \citet{errica2022fair},  and report the mean accuracy and standard deviation.
The results for the baselines DGCNN  \citep{wang2019dynamic}, DiffPool \citep{ying2019hierarchical},   Edge-Conditioned Convolution (ECC) \citep{simonovsky2017dynamic}, GIN, GraphSAGE are from \citet{errica2022fair}. We also include $\text{CGMM}$ \citep{JMLR:v21:19-470}, $\text{ICGMM}_f$ \citep{castellana22a}, SPN($k=5$) \citep{AbboudDC22}
and $\text{GSPN}$ \citep{errica2023tractable} as more recent baselines. OOR (Out of Resources) implies extremely long training time or GPU memory usage. We use Adam optimizer and StepLR learn rate scheduler, and report all hyperparameters in the appendix (\cref{tab:graph-classification-hps}).

\begin{table*}[t!]
\caption{Results on graph classification. Top three models are colored by \red{First}, \blue{Second}, \gray{Third}.}
  \centering
  \begin{tabular}{l@{\hspace{2pt}}c@{\hspace{5pt}}c@{\hspace{5pt}}c@{\hspace{5pt}}c@{\hspace{5pt}}c@{\hspace{5pt}}c@{\hspace{5pt}}c@{\hspace{5pt}}}
    \toprule
    & IMDB-B & IMDB-M & REDDIT-B & REDDIT-M & NCI1 & PROTEINS & ENZYMES \\
    \midrule
    DGCNN & 69.2 \stdfont{$\pm$ 3.0}            & 45.6 \stdfont{$\pm$ 3.4}            & 87.8 \stdfont{$\pm$ 2.5} & 49.2 \stdfont{$\pm$ 1.2} & 76.4 \stdfont{$\pm$ 1.7}          & 72.9 \stdfont{$\pm$ 3.5}        & 38.9 \stdfont{$\pm$ 5.7} \\
    DiffPool & 68.4 \stdfont{$\pm$ 3.3}            & 45.6 \stdfont{$\pm$ 3.4}            & {89.1} \stdfont{$\pm$ 1.6} & 53.8 \stdfont{$\pm$ 1.4} & 76.9 \stdfont{$\pm$ 1.9}          & \blue{73.7} \stdfont{$\pm$ 3.5}  & 59.5 \stdfont{$\pm$ 5.6} \\
    ECC & 67.7 \stdfont{$\pm$ 2.8}            & 43.5 \stdfont{$\pm$ 3.1}            & OOR & OOR & 76.2 \stdfont{$\pm$ 1.4}       & 72.3 \stdfont{$\pm$ 3.4} & 29.5 \stdfont{$\pm$ 8.2} \\
    GIN & \gray{71.2} \stdfont{$\pm$ 3.9}     & \gray{48.5} \stdfont{$\pm$ 3.3}     & \gray{89.9} \stdfont{$\pm$ 1.9} & \blue{56.1} \stdfont{$\pm$ 1.7} & \blue{80.0} \stdfont{$\pm$ 1.4}   & \gray{73.3} \stdfont{$\pm$ 4.0} & {59.6} \stdfont{$\pm$ 4.5} \\
    GraphSAGE & 68.8 \stdfont{$\pm$ 4.5}            & 47.6 \stdfont{$\pm$ 3.5}            & 84.3 \stdfont{$\pm$ 1.9} & 50.0 \stdfont{$\pm$ 1.3} & 76.0 \stdfont{$\pm$ 1.8}          & 73.0 \stdfont{$\pm$ 4.5} & 58.2 \stdfont{$\pm$ 6.0} \\
    $\text{CGMM}$ & - & - & 88.1 \stdfont{$\pm$ 1.9} & 52.4 \stdfont{$\pm$ 2.2} &  76.2 \stdfont{$\pm$ 2.0} & - &-\\
    $\text{ICGMM}_f$ & \blue{71.8} \stdfont{$\pm$ 4.4} & \blue{49.0} \stdfont{$\pm$ 3.8} & \red{91.6} \stdfont{$\pm$ 2.1} & \gray{55.6} \stdfont{$\pm$ 1.7} &  76.4 \stdfont{$\pm$ 1.4} & 73.2 \stdfont{$\pm$ 3.9} &  - \\
    SPN($k=5$) & - & - & - & - & 78.6 \stdfont{$\pm$ 1.7} &  \red{74.2} \stdfont{$\pm$ 2.7} &  \red{69.4} \stdfont{$\pm$ 6.2} \\
    $\text{GSPN}$ & - & - & \blue{90.5} \stdfont{$\pm$ 1.1} & {55.3} \stdfont{$\pm$ 2.0} & 76.6 \stdfont{$\pm$ 1.9} & - & - \\
    \midrule
    $\cosum$   & 70.8 \stdfont{$\pm$ 3.3}   & \gray{48.5} \stdfont{$\pm$ 4.0}   & 88.6 \stdfont{$\pm$ 2.2} & 53.6 \stdfont{$\pm$ 2.3} & \red{80.6} \stdfont{$\pm$ 1.1} & 73.1 \stdfont{$\pm$ 2.3} & \gray{65.7} \stdfont{$\pm$ 4.9} \\
    $\comean$  & \red{72.2} \stdfont{$\pm$ 4.1}   & \red{49.9} \stdfont{$\pm$ 4.5}   & \blue{90.5} \stdfont{$\pm$ 1.9} & \red{56.3} \stdfont{$\pm$ 2.1} & \gray{79.4} \stdfont{$\pm$ 0.7} & 71.3 \stdfont{$\pm$ 2.0} & \blue{68.3} \stdfont{$\pm$ 5.7} \\
    \bottomrule
  \end{tabular}
  \label{tab:graph-classification}
\end{table*}

\textbf{Results.} $\cognn$ models achieve the highest accuracy on three datasets in \Cref{tab:graph-classification} and remain competitive on the other datasets. $\cognn$ yield these performance improvements, despite using relatively simple action and environment networks, which is intriguing as \cognns unlock a large design space which includes a large class of model variations.

\subsection{Homophilic Node Classification}
\begin{wraptable}{r}{0.45\textwidth} 
\caption{Results on homophilic datasets. Top three models are colored by \red{First}, \blue{Second}, \gray{Third}.}
\vspace{-0.5em}
  \centering
  \begin{tabular}{l@{\hspace{2pt}}ccc}
    \toprule
    & pubmed & cora\\
    \midrule
    MLP        & 87.16 \stdfont{$\pm$ 0.37}        & 75.69 \stdfont{$\pm$ 2.00}\\
    GCN        & 88.42 \stdfont{$\pm$ 0.50}        & \gray{86.98} \stdfont{$\pm$ 1.27}\\
    GraphSAGE  & 88.45 \stdfont{$\pm$ 0.50}        & 86.90 \stdfont{$\pm$ 1.04}\\
    GAT        & 87.30 \stdfont{$\pm$ 1.10}        & 86.33 \stdfont{$\pm$ 0.48}\\
    Geom-GCN   & 87.53 \stdfont{$\pm$ 0.44}        & 85.35 \stdfont{$\pm$ 1.57}\\
    GCNII      & \red{90.15} \stdfont{$\pm$ 0.43}  & \red{88.37} \stdfont{$\pm$ 1.25}\\
    \midrule
    $\sumgnn$  & 88.58 \stdfont{$\pm$ 0.57}        & 84.80 \stdfont{$\pm$ 1.71} \\
    $\meangnn$ & 88.66 \stdfont{$\pm$ 0.44}        & 84.50 \stdfont{$\pm$ 1.25} \\
    \midrule
    $\cosum$   & 89.39 \stdfont{$\pm$ 0.39} & 86.43 \stdfont{$\pm$ 1.28}\\
    $\comean$  & \blue{89.60} \stdfont{$\pm$ 0.42} & 86.53 \stdfont{$\pm$ 1.20}\\
    $\cogcn$  & \gray{89.51} \stdfont{$\pm$ 0.88} & \blue{87.44} \stdfont{$\pm$ 0.85}\\
    \bottomrule
  \end{tabular}
  \label{tab:homophilic}
\end{wraptable}
In this experiment, we evaluate $\cognns$ on the homophilic node classification benchmarks cora and pubmed \citep{sen2008collective}.

\textbf{Setup.} We assess $\meangnn$, $\sumgnn$ and their corresponding $\cognns$ counterparts $\comean$ and $\cosum$ on the homophilic graphs and their 10 fixed splits provided by \citet{pei2020geomgcn}, where we report the mean accuracy, standard deviation and the accuracy gain due to the application of $\cognn$. We also use the results provided by \citet{bodnar2023neural} for the classical baseline: GCN, GraphSAGE, GAT, Geom-GCN \citep{pei2020geomgcn} and GCNII.

\textbf{Results.} \cref{tab:homophilic} illustrates a modest performance increase of 1-2\% across all datasets when transitioning from $\sumgnn$, $\meangnn$, and GCN to their respective $\cognn$ counterparts. These datasets are highly homophilic, but $\cognns$ nonetheless show improvements on these datasets (even though, modest) compared to their environment/action network architectures. 

\begin{figure*}[t!]
    \centering
    \includegraphics[width=0.45\linewidth]{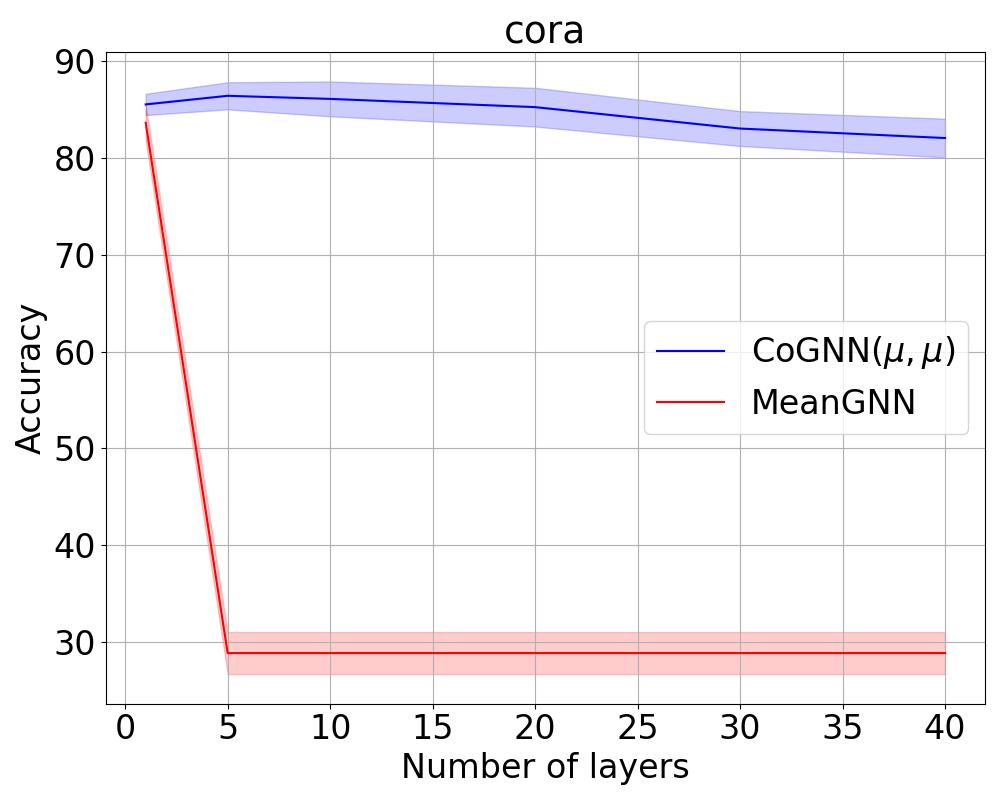}
    \hfill
    \includegraphics[width=0.45\linewidth]{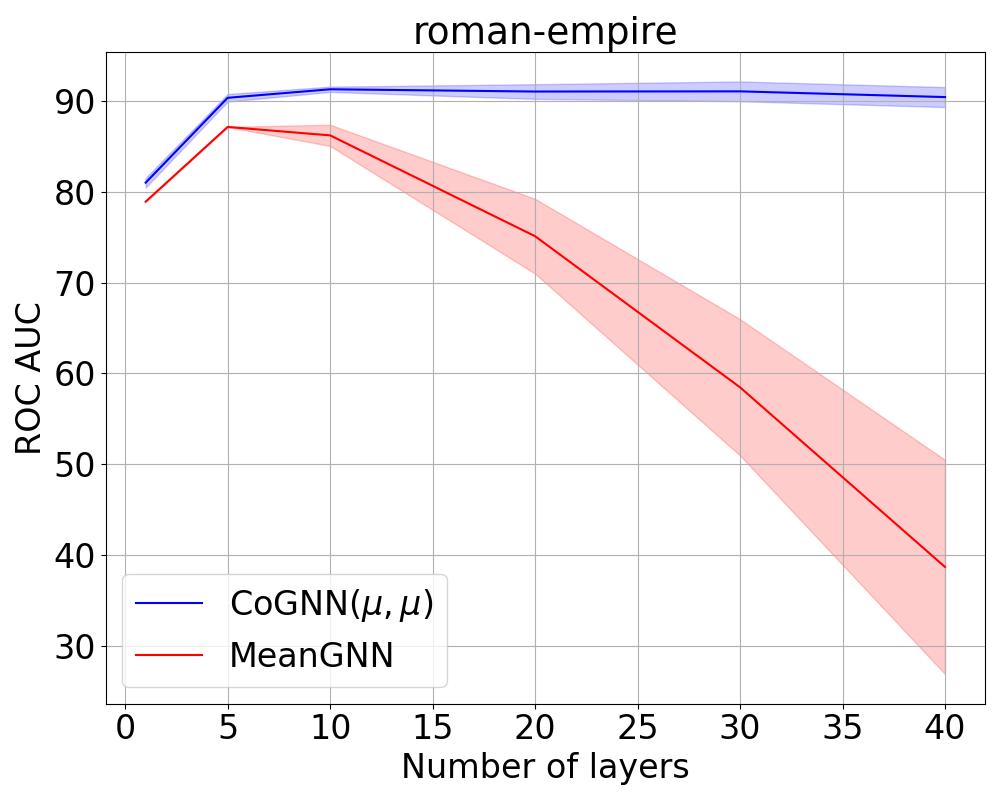}
    \caption{The accuracy of $\comean$ and $\meangnn$ on cora (\textbf{left}) and on roman-empire (\textbf{right}) for an increasing number of layers.}
    \label{fig:oversmoothing}
\end{figure*}
\subsection{Over-smoothing Experiments}
\label{app:oversmoothing}
\cref{subsec:conceptual} explains that \cognns can mitigate the over-smoothing phenomenon, through the choice of $\broadcastlong$ or $\isolatelong$ actions. To validate this, we experiment with an increasing number of layers of $\comean$ and $\meangnn$ over the cora and roman-empire datasets.

\textbf{Setup.} We evaluate $\comean$ and $\meangnn$ over the cora and roman-empire datasets, following the 10 data splits of \citet{pei2020geomgcn} and \citet{platonov2023critical}, respectively. We report the accuracy and standard deviation.

\textbf{Results.} \cref{fig:oversmoothing} indicates that the performance is generally retained for deep models and that \cognns are effective in alleviating the over-smoothing phenomenon even though their base GNNs suffer from performance deterioration already with a few layers.

\section{Runtime Analysis}
\label{app:runtime}

\begin{figure}[ht!]
    \centering
    \includegraphics[width=0.7\textwidth]{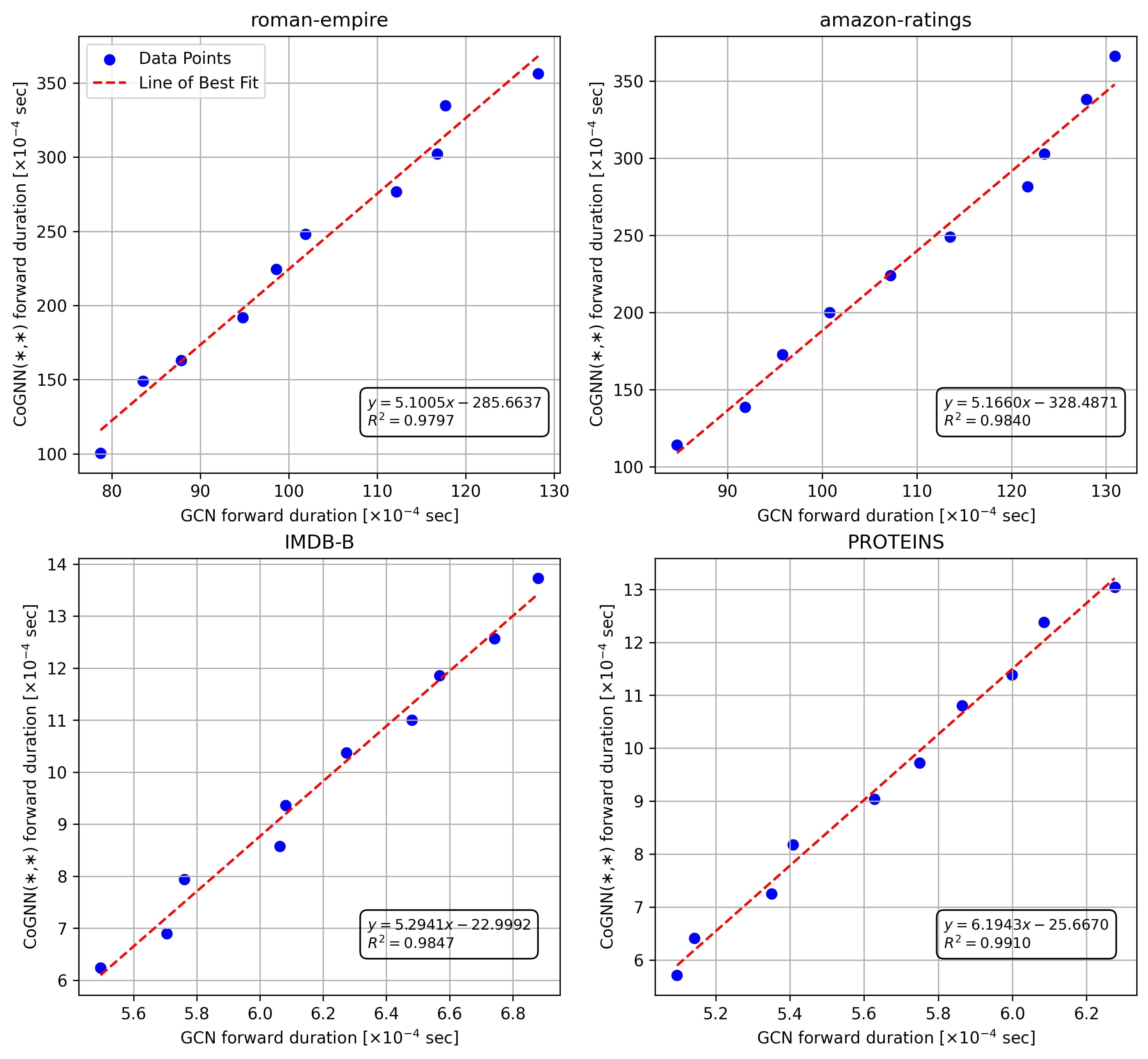}
    \caption{Empirical runtimes: $\cogcn$ forward pass duration as a function of GCN forward pass duration.}
    \label{fig:runtime}
\end{figure}

Consider a GCN model with $L$ layers and a hidden dimension of $d$ on an input graph $G=(V,E,\mX)$. \citet{Wu2019ACS} has shown the time complexity of this model to be $\mathcal{O}(Ld(|E|d + |V|))$.  To extend this analysis to \cognns, let us consider a $\cogcn$ architecture composed of: 
\begin{enumerate}[$\bullet$]
    \item a GCN environment network $\env$ with $L_{\env}$ layers and hidden dimension of $d_{\env}$, and
    \item a GCN action network $\action$ with $L_{\action}$ layers and hidden dimension of $d_{\action}$. 
\end{enumerate}

A single $\cognn$ layer first computes the actions for each node by feeding node representations through the action network $\action$, which is then used in the aggregation performed by the environment layer. This means that the time complexity of a single $\cognn$ layer is 
$\mathcal{O}(L_{\action}d_{\action}(|E|d_{\action} + |V|) + d_{\env}(|E|d_{\env} + |V|))$. The time complexity of the whole $\cognn$ architecture is then  $\mathcal{O}(L_{\env}L_{\action}d_{\action}(|E|d_{\action} + |V|) + L_{\env}d_{\env}(|E|d_{\env} + |V|))$.

Typically, the hidden dimensions of the environment network and action network match. In all of our experiments, the depth of the action network $L_{\action}$ is much smaller (typically $\leq 3$) than that of the environment network $L_{\env}$. Therefore, assuming $L_{\action} << L_{\env}$ we get that a runtime complexity of $\mathcal{O}(L_{\env}d_{\env}(|E|d_{\env} + |V|))$, matching the runtime of a GCN model. 

To empirically confirm the efficiency of \cognns, we report in \Cref{fig:runtime} the duration of a forward pass of a $\cogcn$ and GCN with matching hyperparameters across multiple datasets. From \Cref{fig:runtime}, it is evident that the increase in runtime is linearly related to its corresponding base model with $R^2$ values higher or equal to $0.98$ across $4$ datasets from different domains. Note that, for the datasets IMDB-B and PROTEINS, we report the average forward duration for a single graph in a batch. 

\section{Further Details of the Experiments Reported in the Paper}
\label{app:dataset-details}

\subsection{The Gumbel Distribution and the Gumbel-softmax Temperature}
\label{app:gumbel}

The Gumbel distribution is used to model the distribution of the maximum (or the minimum) of a set of random variables. 
\begin{wrapfigure}{r}{0.44\textwidth} 
    \centering
    \includegraphics[width=1\linewidth]{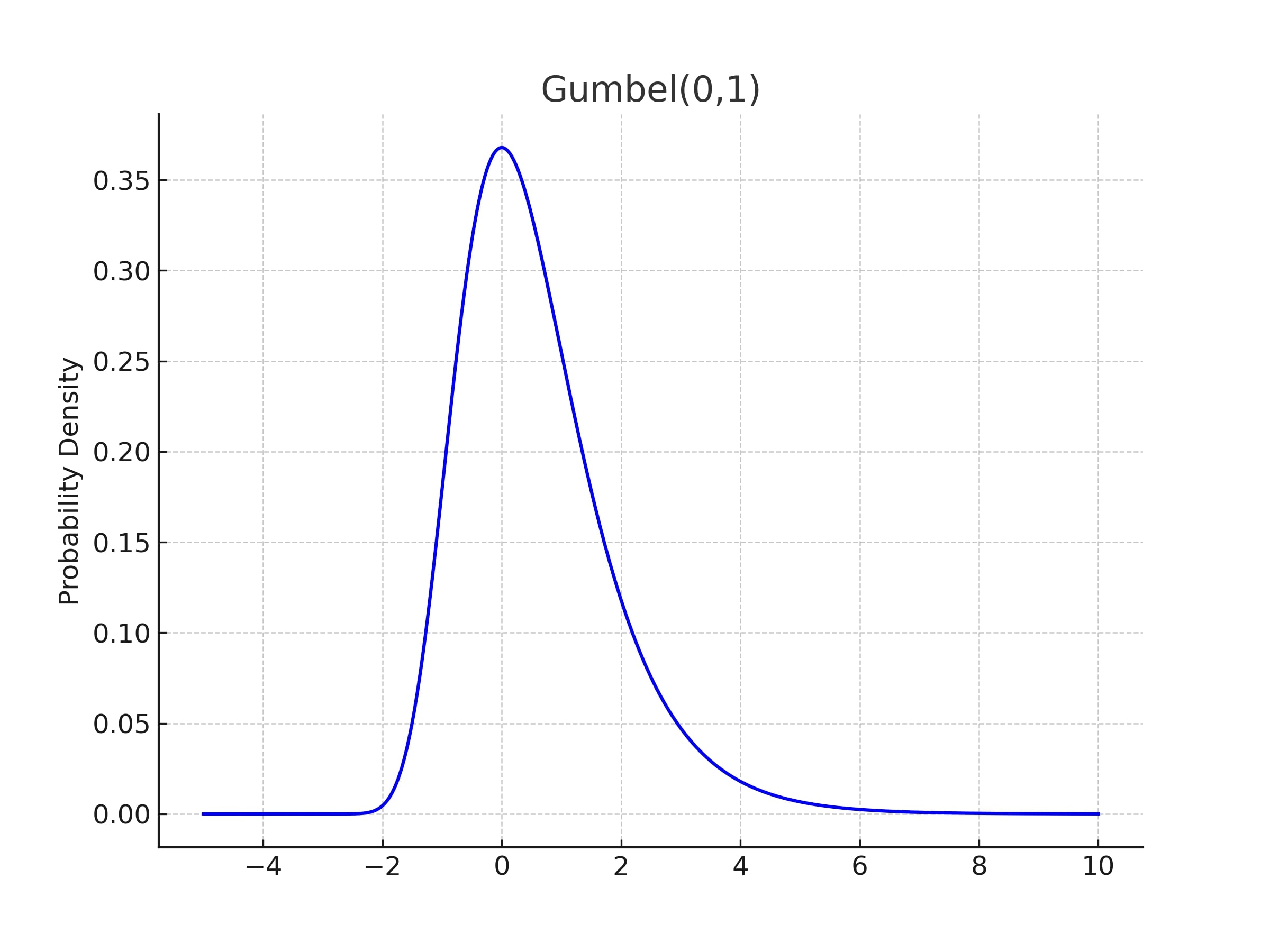}
    \caption{The pdf $f(x) = e^{-x + e^{-x}}$ of $\operatorname{Gumbel}(0,1)$.}
    \label{fig:gumbel_distribution}
\end{wrapfigure}
Its probability density function has a distinctive, skewed shape, with heavy tails, making it a valuable tool for analyzing and quantifying the likelihood of rare and extreme occurrences. 
By applying the Gumbel distribution to the logits or scores associated with discrete choices, the Gumbel-Softmax estimator transforms them into a probability distribution over the discrete options. 
The probability density function of a variable that follows $X\sim \operatorname{Gumbel}(0,1)$ is 
$f(x) = e^{-x + e^{-x}}$ (\Cref{fig:gumbel_distribution}).

The Straight-through Gumbel-softmax estimator is known to benefit from learning an inverse-temperature before sampling an action, which we use in our experimental setup. 
For a given graph $G=(V, E, \mX)$ the inverse-temperature of node $v\in V$ is estimated by applying a bias-free linear layer $L:\sR^{d}\rightarrow\sR$ to the intermediate representation $\vh \in \sR^d$. To ensure the temperature is positive, an approximation of the ReLU function with a bias hyperparameter $\tau\in\sR$ is subsequently applied:
\begin{equation*}
    \frac{1}{\tau\left(\vh\right)} = \log \left(1+\exp \left(\omega^T\vh\right)\right) + \tau_0
\end{equation*}
where $\tau_0$ controls the maximum possible temperature value.

\subsection{Dataset Statistics}
The statistics of the real-world long-range, node-based, and graph-based benchmarks used can be found in \cref{tab:hetero-node-classification-stats,tab:long-range-stats,tab:graph-classification-stats,tab:homo-node-classification-stats}.

\begin{table}[h!]
\caption{Statistics of the heterophilic node classification benchmarks.}
  \centering
  \begin{tabular}{l@{\hspace{6pt}}c@{\hspace{6pt}}c@{\hspace{6pt}}c@{\hspace{6pt}}c@{\hspace{6pt}}c@{\hspace{6pt}}}
    \toprule
      & roman-empire & amazon-ratings & minesweeper & tolokers & questions \\
      \midrule
      \# nodes & 22662 & 24492 & 10000 & 11758 & 48921 \\
      \# edges & 32927 & 93050 & 39402 & 519000 & 153540 \\
      \# node features & 300 & 300 & 7 & 10 & 301 \\
      \# classes & 18 & 5 & 2 & 2 & 2\\
      edge homophily & 0.05 & 0.38 & 0.68 & 0.59 & 0.84 \\
      metrics & ACC & ACC & AUC-ROC & AUC-ROC & AUC-ROC \\
    \bottomrule
  \end{tabular}
  \label{tab:hetero-node-classification-stats}
\end{table}

\begin{table}[h!]
\caption{Statistics of the long-range graph benchmarks (LRGB).}
  \centering
  \begin{tabular}{lc}
    \toprule
      & Peptides-func  \\
      \midrule
    \# graphs & 15535 \\
    \# average nodes & 150.94 \\
    \# average edges & 307.30 \\
    \# classes & 10   \\
    metrics & AP  \\
    \bottomrule
  \end{tabular}
  \label{tab:long-range-stats}
\end{table}

\begin{table}[h!]
\caption{Statistics of the graph classification benchmarks.}
  \centering
  \begin{tabular}{lcccccc}
    \toprule
      & IMDB-B & IMDB-M & REDDIT-B & NCI1 & PROTEINS & ENZYMES\\
      \midrule
    \# graphs & 1000 & 1500 & 2000 & 4110 & 1113 & 600  \\
    \# average nodes & 19.77 & 13.00 & 429.63 & 29.87 & 39.06 & 32.63 \\
    \# average edges & 96.53 & 65.94 & 497.75 & 32.30 & 72.82 & 64.14 \\
    \# classes & 2 & 3 & 2 & 2 & 2 & 6 \\
    metrics & ACC & ACC & ACC & ACC & ACC & ACC \\
    \bottomrule
  \end{tabular}
  \label{tab:graph-classification-stats}
\end{table}

\begin{table}[h!]
\caption{Statistics of the homophilic node classification benchmarks.}
  \centering
  \begin{tabular}{l@{\hspace{6pt}}c@{\hspace{6pt}}c@{\hspace{6pt}}}
    \toprule
      & pubmed & cora \\
      \midrule
      \# nodes & 18717 & 2708 \\
      \# edges & 44327 & 5278 \\
      \# node features & 500 & 1433 \\
      \# classes & 3 & 6\\
      edge homophily & 0.80 & 0.81 \\
      metrics & ACC & ACC \\
    \bottomrule
  \end{tabular}
  \label{tab:homo-node-classification-stats}
\end{table}

\subsection{\rootn: Dataset Generation}
\label{app:synthetic_generation}
In \cref{subsec:synthetic}, we compare \cognns to a class of MPNNs on a dedicated synthetic dataset \rootn in order to assess the model's ability to redirect the information flow. \rootn consists of 3000 trees of depth 2 with random node features of dimension $d = 5$ which is generated as follows:
\begin{enumerate}[$\bullet$,leftmargin=.6cm]
    \item \textbf{Features}: Each feature is independently sampled from a uniform distribution $U[-2, 2]$.
    \item \textbf{Level-1 Nodes}: The number of nodes in the first level of each tree in the train, validation, and test set is sampled from a uniform distribution $U[3, 10]$, $U[5, 12]$, and $U[5, 12]$ respectively. Then, the degrees of the level-1 nodes are sampled as follows:
    \begin{enumerate}[$\bullet$]
        \item The number of level-1 nodes with a degree of $6$ is sampled independently from a uniform distribution $U[1, 3], U[3, 5], U[3, 5]$  for the train, validation, and test set, respectively.
        \item The degree of the remaining level-1 nodes are sampled from the uniform distribution $U[2, 3]$.
    \end{enumerate}
\end{enumerate}
We use a train, validation, and test split of equal size.

\subsection{Hyperparameters for all Experiments}
\label{app:hyperparameters}
In \cref{tab:synthetic-hps,tab:long-range-hps,tab:heterophilic-node-classification-hps,tab:graph-classification-hps,tab:homophilic-node-classification-hps}, we report the hyperparameters used in our experiments.

\begin{table}[ht!]
\caption{Hyperparameters used for \rootn and \cycles.}
  \centering
  \begin{tabular}{l@{\hspace{6pt}}c@{\hspace{6pt}}c@{\hspace{6pt}}}
    \toprule
      &  \rootn & \cycles \\
      \midrule
      $\env$ \# layers & 1  & 2\\
      $\env$ dim       & 16, 32 & 32\\
      $\action$ \# layers & 1, 2 & 6\\
      $\action$ dim       & 8, 16 & 32\\
      learned temp        & \checkmark & - \\
      temp & - & 1\\
      $\tau_0$ & 0.1 & -\\
      \# epochs & 10000 & 1000\\
      \midrule
      dropout & 0 & 0\\
      learn rate & $10^{-3}$ & $ 10^{-3}$\\
      batch size & - & 14\\
      pooling & - & sum\\
    \bottomrule
  \end{tabular}
  \label{tab:synthetic-hps}
\end{table}

\begin{table}[ht!]
\caption{Hyperparameters used for the heterophilic node classification benchmarks.}
  \centering
  \begin{tabular}{l@{\hspace{4pt}}c@{\hspace{4pt}}c@{\hspace{4pt}}c@{\hspace{4pt}}c@{\hspace{4pt}}c@{\hspace{4pt}}}
    \toprule
      & roman-empire & amazon-ratings & minesweeper & tolokers & questions \\
      \midrule
      $\env$ \# layers               & 5-12 & 5-10 & 8-15 & 5-10 & 5-9\\
      $\env$ dim                     & 128,256,512 & 128,256 & 32,64,128 & 16,32 & 32,64 \\
      $\action$ \# layers            & 1-3 & 1-6 & 1-3 & 1-3 & 1-3\\
      $\action$ dim                  & 4,8,16 & 4,8,16,32 & 4,8,16,32,64 & 4,8,16,32 & 4,8,16,32\\
      learned temp                   & \checkmark & \checkmark & \checkmark & \checkmark & \checkmark\\
      $\tau_0$                       & 0,0.1 & 0,0.1 & 0,0.1 & 0,0.1 & 0,0.1\\
      \# epochs                      & 3000 & 3000 & 3000 & 3000 & 3000 \\
      \midrule
      dropout                        & 0.2 & 0.2 & 0.2 & 0.2 & 0.2\\
      learn rate                     & $3\cdot 10^{-3},3\cdot 10^{-5}$& $3\cdot 10^{-4},3\cdot 10^{-5}$&$3\cdot 10^{-3},3\cdot 10^{-5}$&$3\cdot 10^{-3}$ & $ 10^{-3}, 10^{-2}$  \\
      activation function            & GeLU     & GeLU  & GeLU  & GeLU  & GeLU            \\
      skip connections               & \checkmark & \checkmark  & \checkmark  & \checkmark  & \checkmark         \\
      layer normalization            & \checkmark   & \checkmark  & \checkmark  & \checkmark  & \checkmark        \\
    \bottomrule
  \end{tabular}
  \label{tab:heterophilic-node-classification-hps}
\end{table}

\begin{table}[ht!]
\caption{Hyperparameters used for the long-range graph benchmarks (LRGB).}
  \centering
  \begin{tabular}{l@{\hspace{6pt}}c@{\hspace{6pt}}}
    \toprule
      &  Peptides-func\\
      \midrule
      $\env$ \# layers                    & 5-9 \\
      $\env$ dim                     & 200,300 \\
      $\action$ \# layers                 & 1-3\\
      $\action$ dim                      & 8,16,32\\
      learned temp                   & \checkmark \\
      $\tau_0$                             & 0.5\\
      \# epochs                           & 500\\
      \midrule
      dropout                             & 0\\
      learn rate                         & $3\cdot 10^{-4}, 10^{-3}$\\
      \# decoder layer              &  2,3 \\
      \# warmup epochs               &  5\\
      positional encoding       & LapPE, RWSE \\
      batch norm               & \checkmark \\
      skip connections               & \checkmark \\
    \bottomrule
  \end{tabular}
  \label{tab:long-range-hps}
\end{table}

\begin{table}[ht!]
\caption{Hyperparameters used for social networks and proteins datasets.}
  \centering
  \begin{tabular}{l@{\hspace{6pt}}c@{\hspace{6pt}}c@{\hspace{6pt}}c@{\hspace{6pt}}c@{\hspace{6pt}}c@{\hspace{6pt}}c@{\hspace{6pt}}c@{\hspace{6pt}}}
    \toprule
      &IMDB-B        & IMDB-M     & REDDIT-B & REDDIT-M   & NCI1                & PROTEINS & ENZYMES\\
      \midrule
      $\env$ \# layers               & 1             & 1           &  3,6 & 6         & 2,5                   & 3,5        & 1,2\\
      $\env$ dim                     & 32,64        &  64,256          & 128,256 & 64, 128          & 64,128,256         &  64       & 128,256\\
      $\action$ \# layers            & 2             &  3          &  1,2 & 1         & 2                    & 1,2        & 1\\
      $\action$ dim                  & 16,32             & 16           &  16,32 & 16         & 8, 16                 & 8        & 8\\
      learned temp.                  & \checkmark   & \checkmark & \checkmark & \checkmark & \checkmark          & \checkmark & \checkmark\\
      $\tau_0$                       & 0.1           & 0.1        & 0.1 & 0.1       & 0.5                 & 0.5        & 0.5\\
      \# epochs                      & 5000          & 5000       & 5000 & 5000      & 3000                & 3000       & 3000\\
      \midrule
      dropout                        & 0.5           & 0.5        & 0.5 & 0.5       & 0                   & 0          & 0\\
      learn rate                     & $10^{-4}$   & $10^{-3}$    &$10^{-3}$ & $10^{-4}$    & $10^{-3}$,$10^{-2}$  &   $10^{-3}$   & $10^{-3}$ \\
      pooling                        & mean          & mean       & mean & mean      & mean                & mean       & mean\\
      batch size                     & 32            & 32         & 32 & 32        & 32                  & 32         & 32\\
      scheduler step size            & 50            & 50         & 50 & 50        & 50                  & 50         & 50\\
      scheduler gamma                & 0.5           & 0.5        & 0.5 & 0.5       & 0.5                 & 0.5        & 0.5\\
    \bottomrule
  \end{tabular}
  \label{tab:graph-classification-hps}
\end{table}

\draft{\begin{table}[ht!]
\caption{Hyperparameters used for the homophilic node classification benchmarks.}
  \centering
  \begin{tabular}{l@{\hspace{4pt}}c@{\hspace{4pt}}c@{\hspace{4pt}}}
    \toprule
      & pubmed & citeseer \\
      \midrule
      $\env$ \# layers               & 1-3               & 1-3\\
      $\env$ dim                     & 32,64,128   & 32,64,128\\
      $\action$ \# layers            & 1-3                & 1-3\\
      $\action$ dim                  & 4,8,16         & 4,8,16\\
      temperature            & 0.01         & 0.01\\
      $\tau_0$                       & 0.1 & 0.1\\
      \# epochs                      & 2000               & 2000\\
      \midrule
      dropout                        & 0.5                & 0.5\\
      learn rate                     & $5\cdot 10^{-3}$, $10^{-2}$, $5 \cdot 10^{-2}$   & $5\cdot 10^{-3}$, $10^{-2}$, $5 \cdot 10^{-2}$\\
      learn rate decay               & $5\cdot 10^{-6}, 5\cdot 10^{-4}$                  & $5\cdot 10^{-6}, 5\cdot 10^{-4}$\\
      activation function            & ReLU               & ReLU\\
    \bottomrule
  \end{tabular}
  \label{tab:homophilic-node-classification-hps}
\end{table}}

\clearpage
\section{Visualizing the Actions}
\label{app:visu_actions}
We extend the discussion about \cognns dynamic topology over the Minesweeper dataset in \Cref{subsec:visualize} and present the evolution of the graph topology from layer $\ell=1$ to layer $\ell=8$.

\textbf{Setup.}
We train a $10$-layered $\comean$ model and present the evolution of the graph topology from layer $\ell=1$ to layer $\ell=8$. We choose a node (black), and at every layer $\ell$, we depict its neighbors up to distance $10$.  In this visualization, nodes which are mines are shown in red, and other nodes in blue. The features of non-mine nodes (indicating the number of neighboring mines) are shown explicitly whereas the nodes whose features are hidden are labeled with a question mark. For each layer $\ell$, we gray out the nodes whose information cannot reach the black node with the remaining layers available.

\begin{figure}[h]
    \centering
    \includegraphics[width=.6\textwidth]{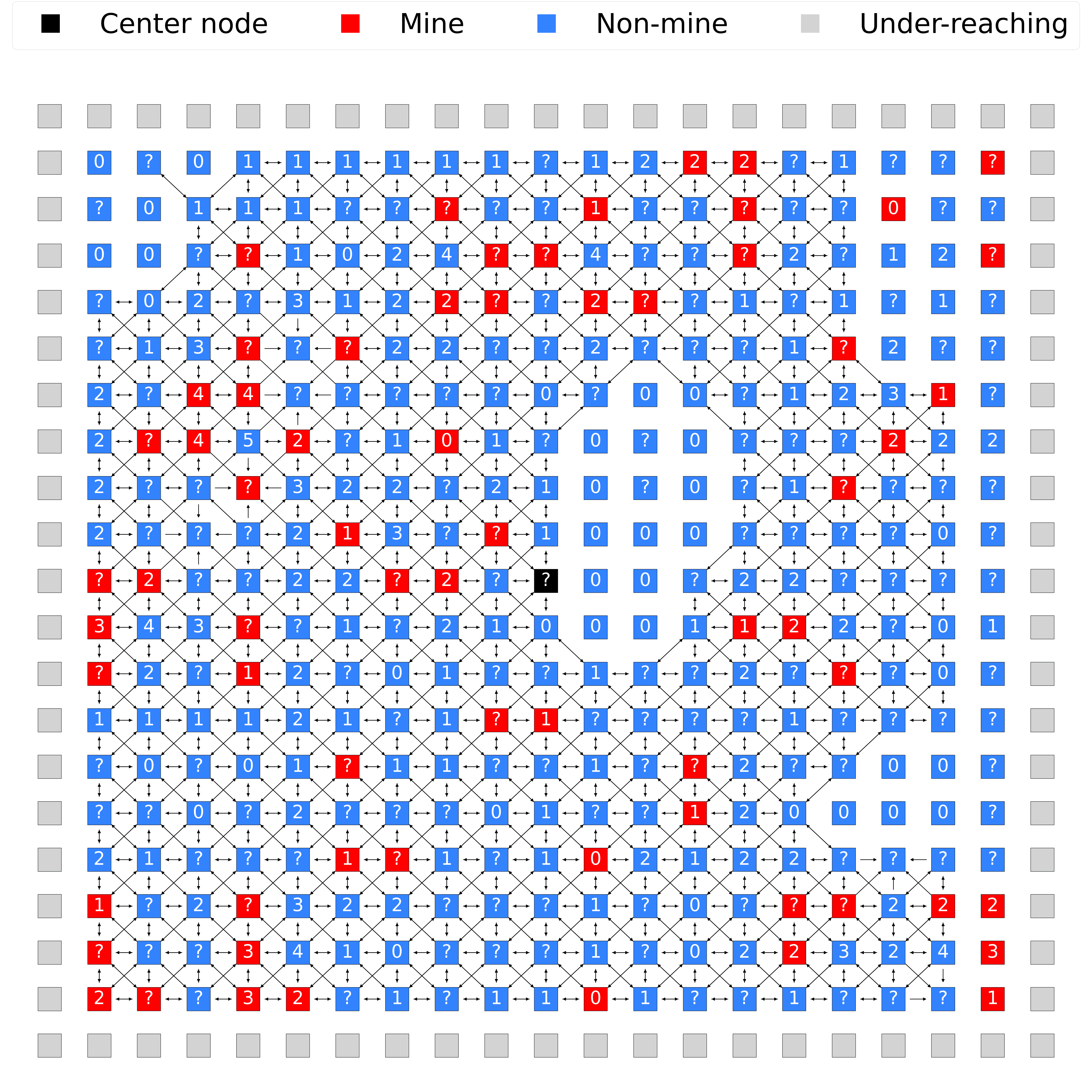}
    \caption{The 10-hop neighborhood at layer $\ell=1$.}
    \label{fig:visualization1}
    \includegraphics[width=.6\textwidth]{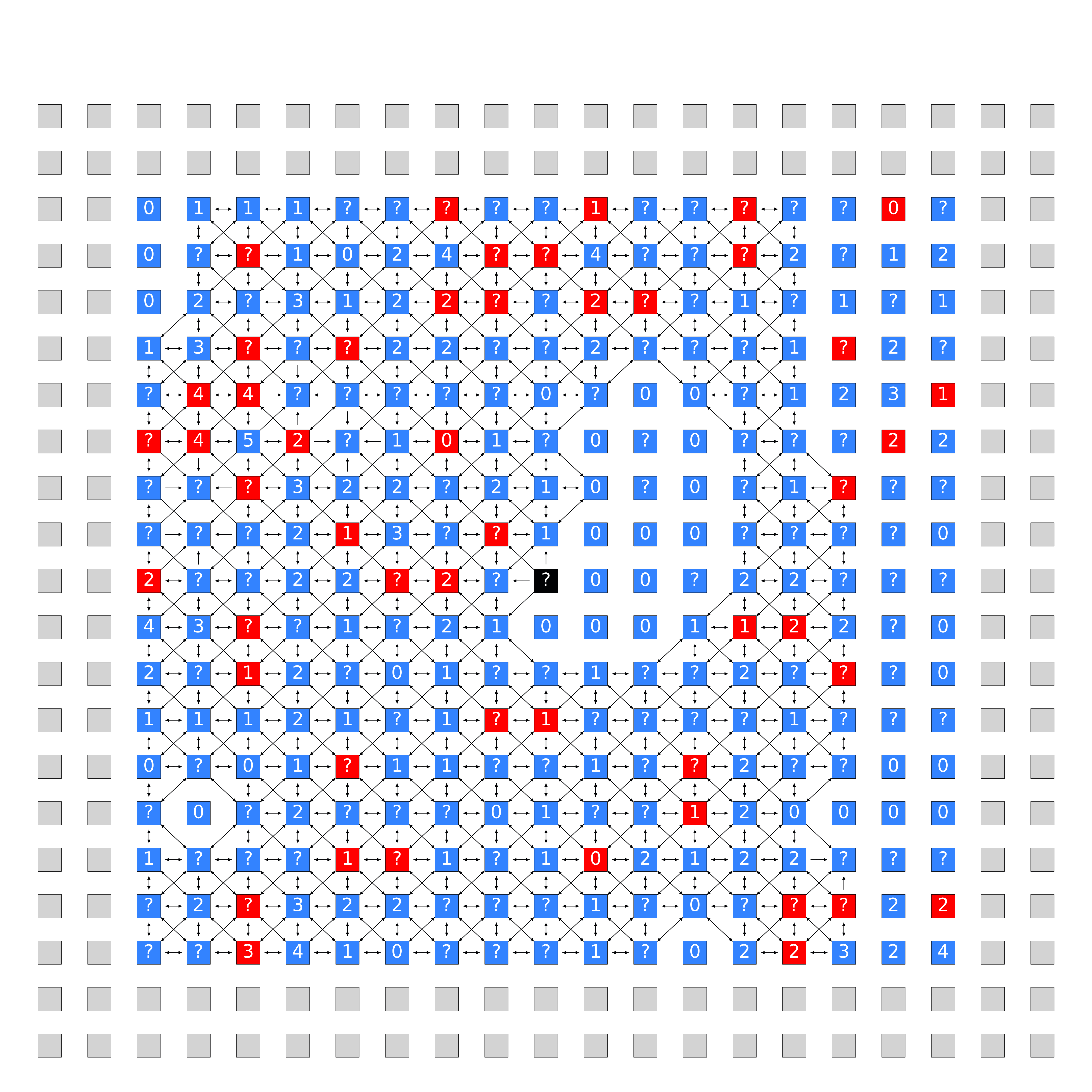}
    \caption{The 10-hop neighborhood at layer $\ell=2$.}
    \label{fig:visualization2}
\end{figure}
\begin{figure}[h]
    \centering
    \includegraphics[width=.6\textwidth]{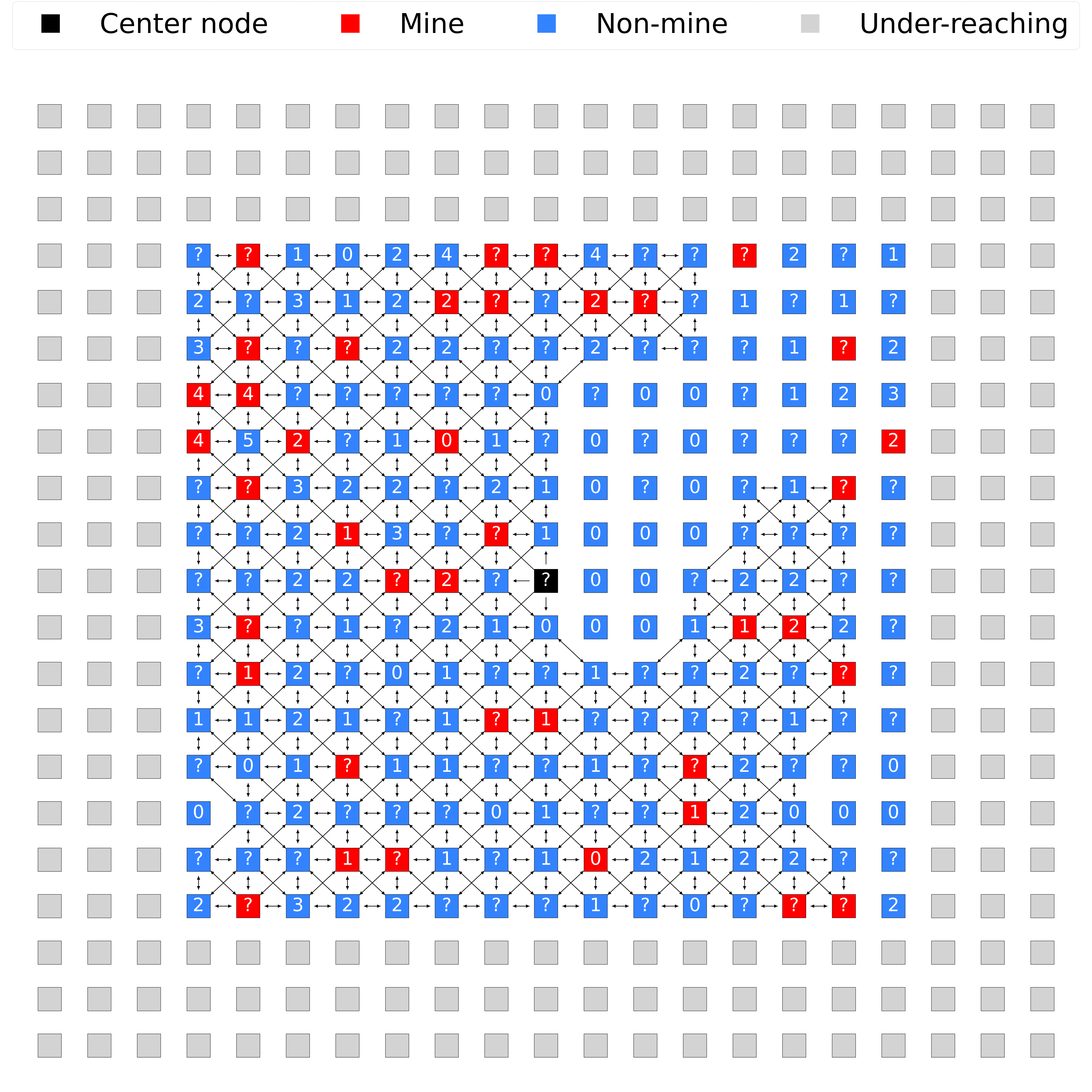}
    \caption{The 10-hop neighborhood at layer $\ell=3$.}
    \label{fig:visualization3}
    \includegraphics[width=.6\textwidth]{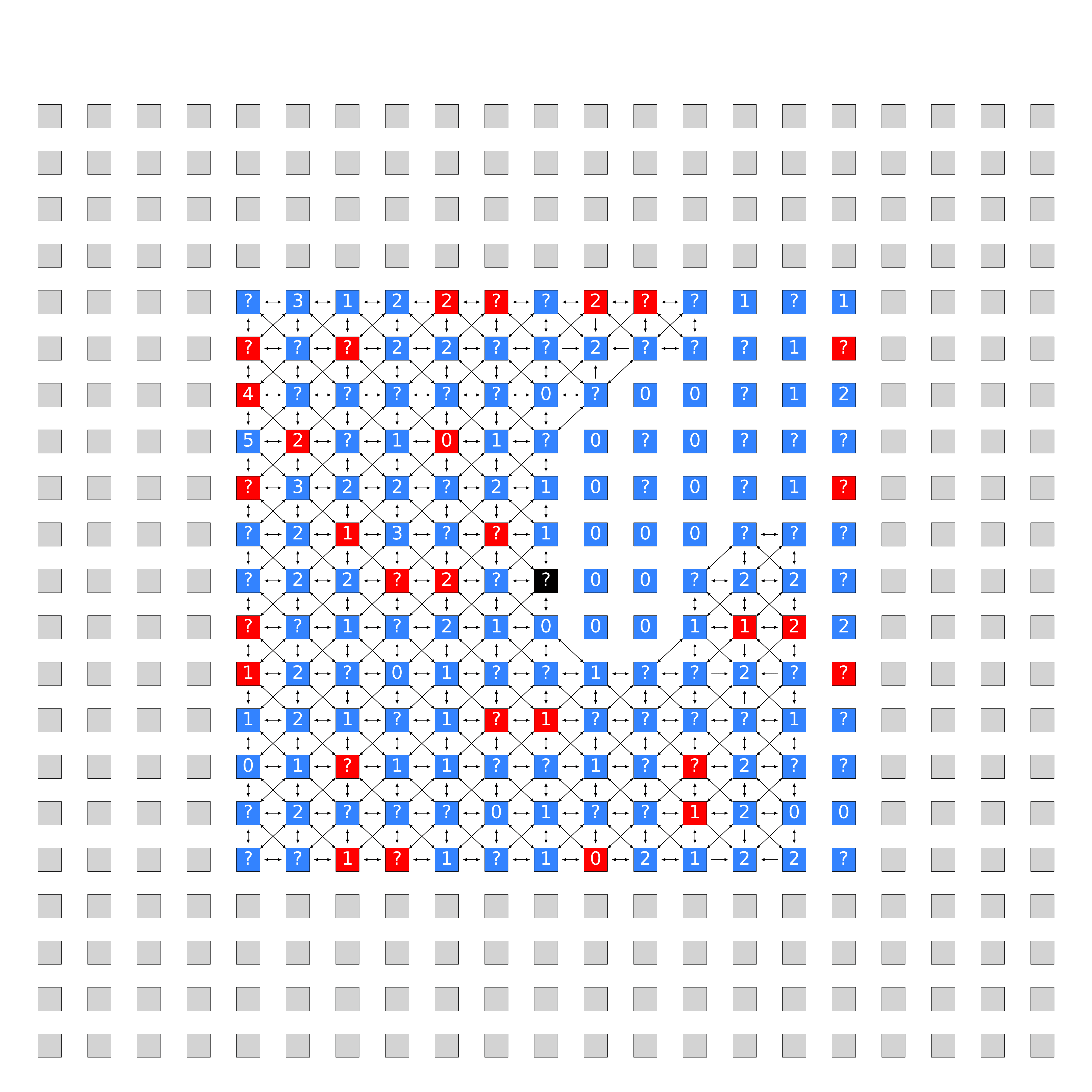}
    \caption{The 10-hop neighborhood at layer $\ell=4$.}
    \label{fig:visualization4}
\end{figure}
\begin{figure}[h]
    \centering
    \includegraphics[width=.6\textwidth]{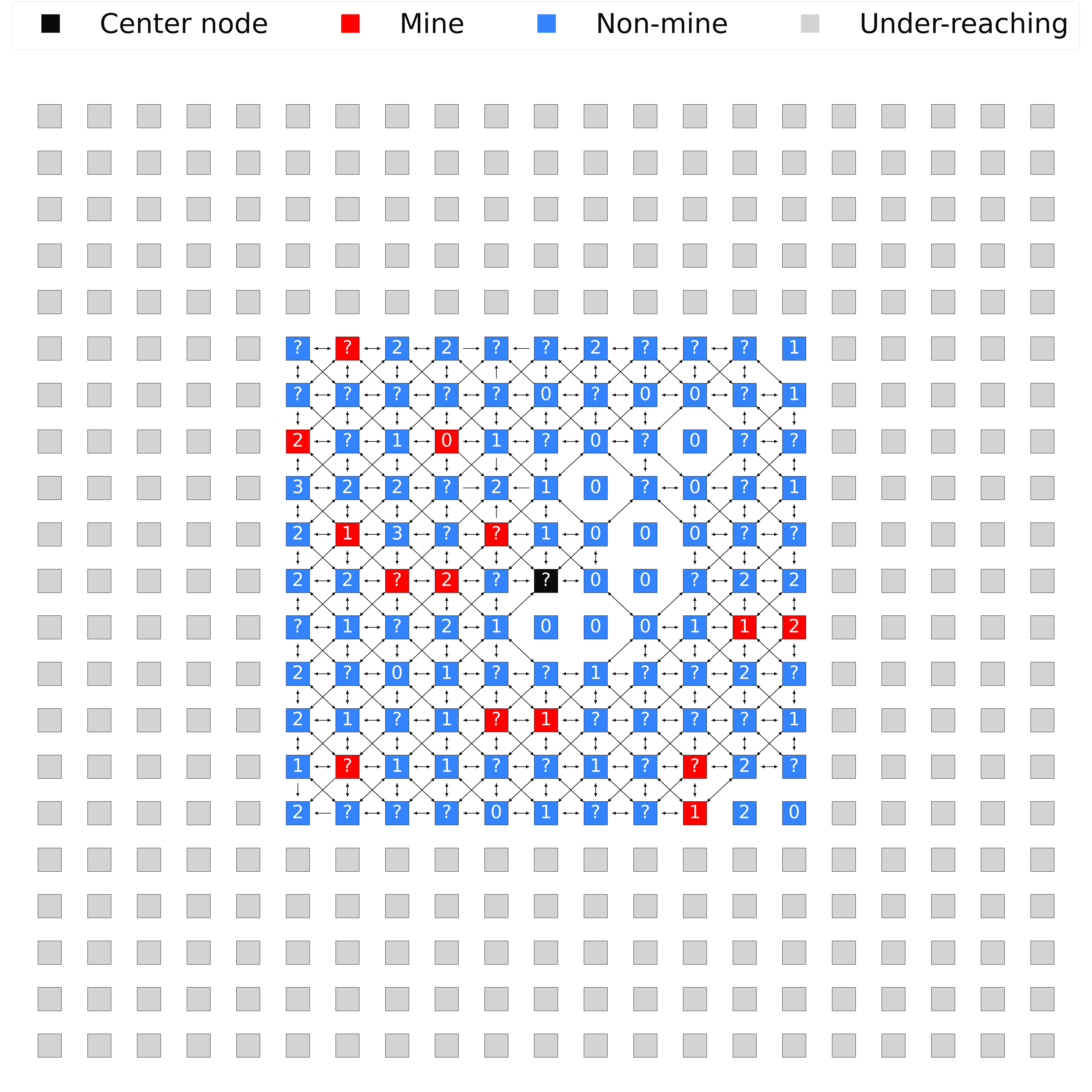}
    \caption{The 10-hop neighborhood at layer $\ell=5$.}
    \label{fig:visualization5}
    \includegraphics[width=.6\textwidth]{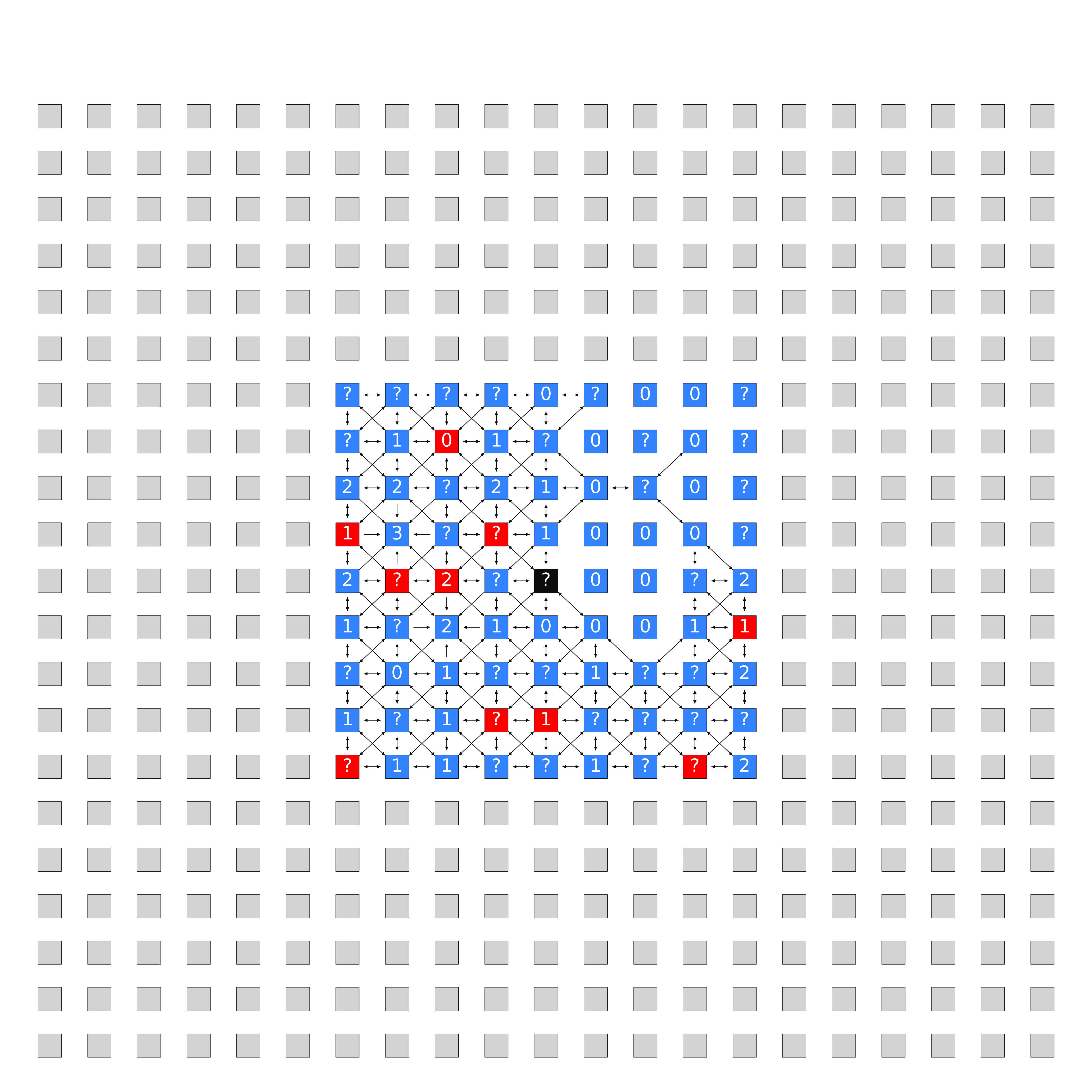}
    \caption{The 10-hop neighborhood at layer $\ell=6$.}
    \label{fig:visualization6}
\end{figure}
\begin{figure}[h]
    \centering
    \includegraphics[width=.6\textwidth]{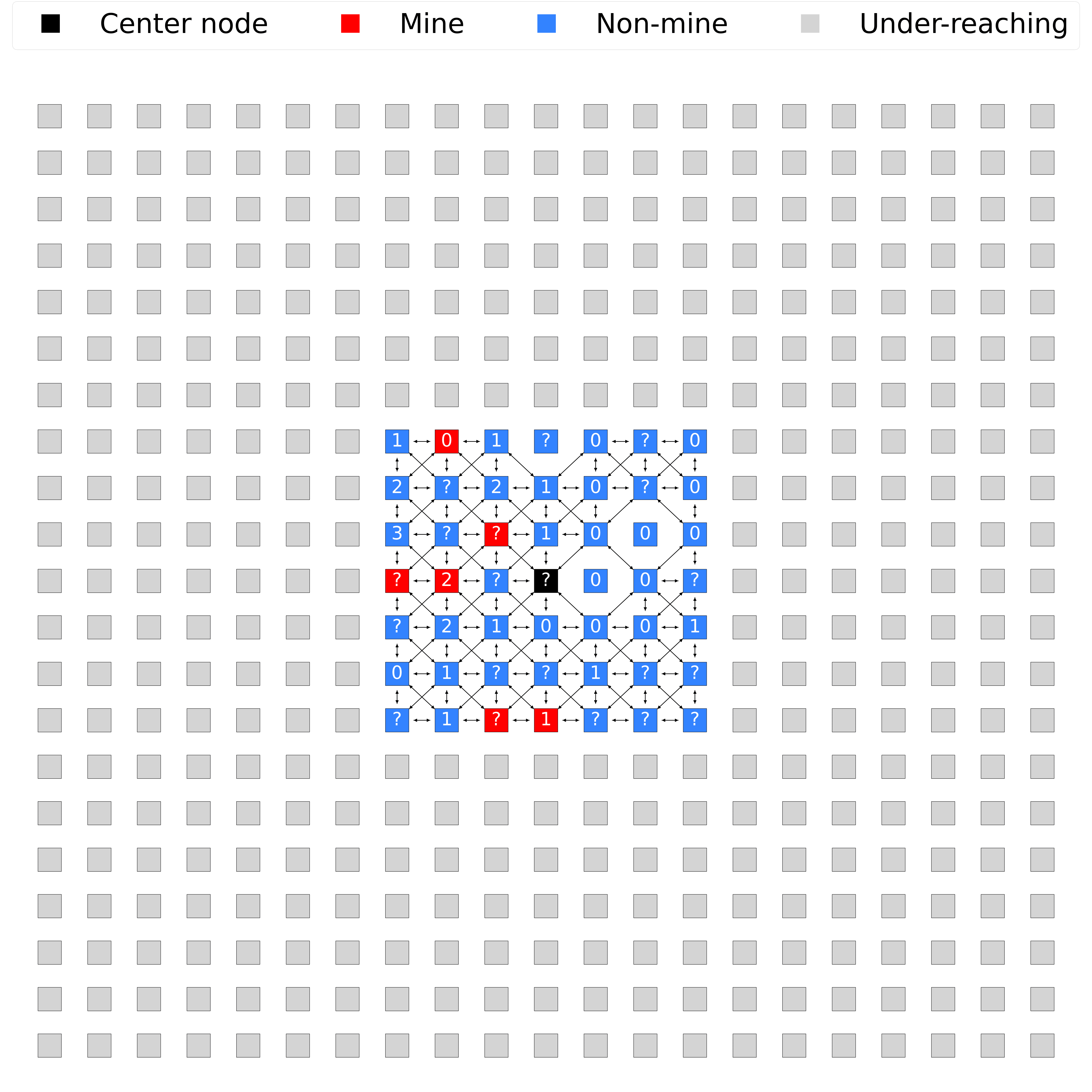}
    \caption{The 10-hop neighborhood at layer $\ell=7$.}
    \label{fig:visualization7}
    \includegraphics[width=.6\textwidth]{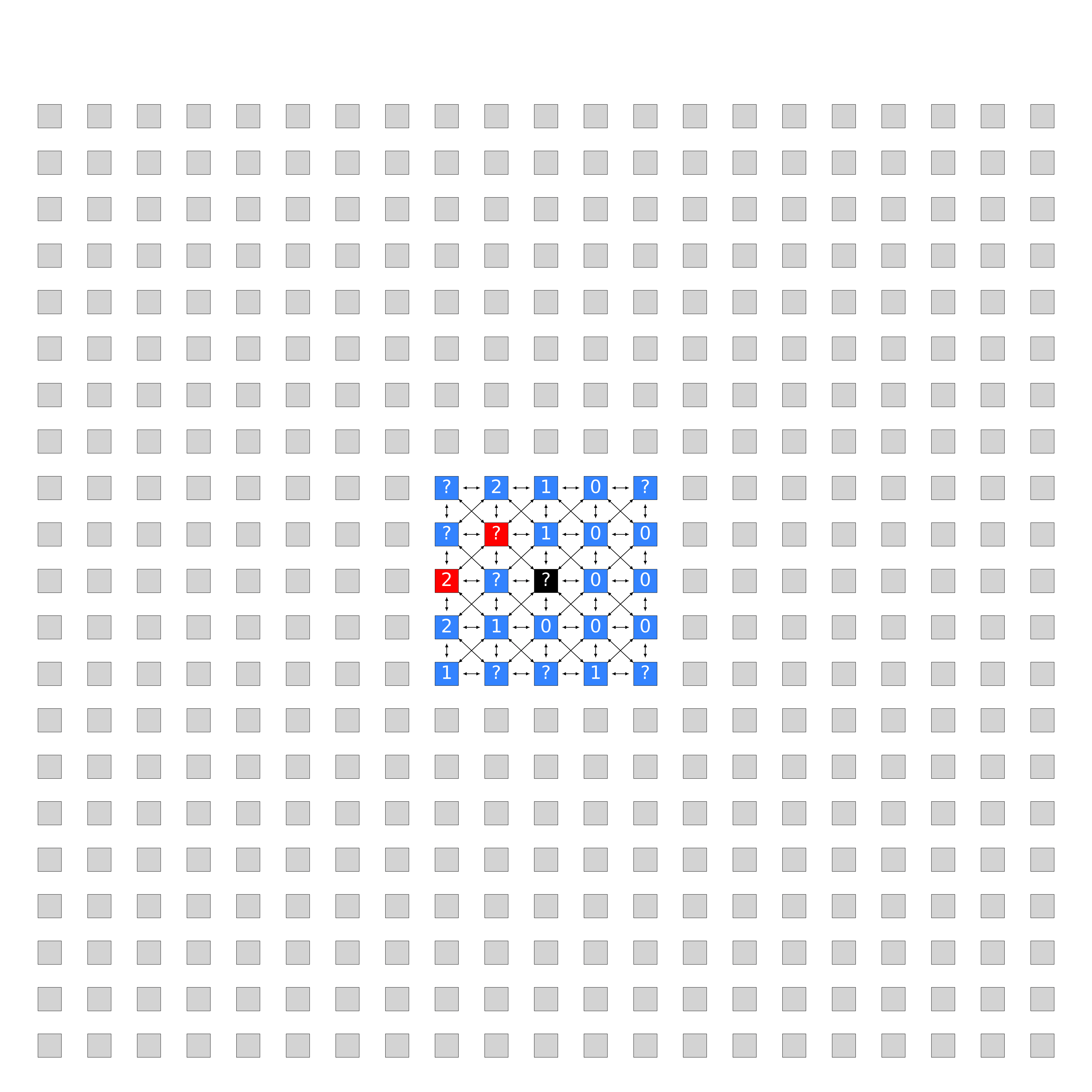}
    \caption{The 10-hop neighborhood at layer $\ell=8$.}
    \label{fig:visualization8}
\end{figure}

\end{document}